\documentclass{article}

\usepackage[preprint]{neurips_2026}

\usepackage[utf8]{inputenc} 
\usepackage[T1]{fontenc}    
\usepackage{hyperref}       
\usepackage{url}            
\usepackage{booktabs}       
\usepackage{amsfonts}       
\usepackage{nicefrac}       
\usepackage{microtype}      
\usepackage{xcolor}         
\usepackage{algpseudocode}
\usepackage{amsmath}
\usepackage{amssymb}
\usepackage{amsthm}
\usepackage{algorithm}
\usepackage{algpseudocode}
\usepackage{mathtools}
\newtheorem{theorem}{Theorem}

\newtheorem{proposition}[theorem]{Proposition}

\usepackage[table]{xcolor}
\usepackage{xcolor}
\usepackage{tcolorbox}
\usepackage{amsmath}
\usepackage{makecell}
\usepackage{multirow}
\usepackage{multicol}
\usepackage{arydshln}
\usepackage{wrapfig}
\tcbuselibrary{breakable,skins}
\definecolor{GCPOgreen}{RGB}{220,245,230}  

\definecolor{myred}{RGB}{241,239,253}    
\definecolor{myblue}{RGB}{235,241,254}   
\definecolor{mygreen}{RGB}{254,242,239}  
\definecolor{myyellow}{RGB}{236,247,244} 
\definecolor{mypurple}{RGB}{254,239,245} 
\definecolor{mygray}{RGB}{235,235,235}   
\definecolor{mywhite}{RGB}{255,255,255}   
\definecolor{GCPOgreen}{RGB}{225,245,230} 
\definecolor{darkgreen}{RGB}{0,128,0}  
\usepackage{booktabs}
\usepackage{colortbl}
\usepackage{xcolor}
\usepackage{graphicx}
\usepackage{geometry}
\title{Breaking \textit{Winner-Takes-All}: Cooperative Policy Optimization Improves Diverse LLM Reasoning}

\author{Haoxuan Chen\textsuperscript{1} \quad Tianming Liang\textsuperscript{1} \quad Wei-Shi Zheng\textsuperscript{1} \quad Jian-Fang Hu\textsuperscript{1}\thanks{Corresponding author.} \\
\textsuperscript{1}ISEE Lab, Sun Yat-sen University \\
\texttt{\{chenhx253, liangtm\}@mail2.sysu.edu.cn, hujf5@mail.sysu.edu.cn} \\
}

\begin{document}

\maketitle

\begin{abstract}

  Reinforcement learning with verifiers (RLVR) has become a central paradigm for improving LLM reasoning, yet popular group-based optimization algorithms like GRPO often suffer from exploration collapse, where the models prematurely converge on a narrow set of high-scoring patterns, lacking the ability to explore new solutions. Recent efforts attempt to alleviate this by adding entropy regularization or diversity bonus. However, these approaches do not change the \textit{winner-takes-all} nature, where rollouts still compete for individual advantage rather than cooperating for maximizing global diversity. In this work, we propose Group Cooperative Policy Optimization (GCPO), which shifts the training paradigm from rollout competition to team cooperation. Specifically, GCPO replaces independent rollout scoring with team-level credit assignment: a rollout is rewarded by how much it contributes to the team's valid solution coverage, rather than its individual accuracy. This coverage is described as a determinant volume over reward-weighted semantic embeddings, where only correct and non-redundant rollouts contribute to this volume. During advantage estimation, GCPO redistributes the collective team reward to each single rollout according to its average marginal contribution to the team. This cooperative training paradigm routes optimization toward non-redundant correct reasoning paths. Experiments across multiple reasoning benchmarks demonstrate that GCPO significantly improves both reasoning accuracy and solution diversity over existing approaches. Code will be released at \href{https://github.com/bradybuddiemarch/gcpo}{https://github.com/bradybuddiemarch/gcpo}.

\end{abstract}

\section{Introduction}
Reinforcement learning with verified reward (RLVR) has emerged as a powerful paradigm for enhancing LLM reasoning~\cite{shao2024deepseekmath, liu2024deepseek}. Group-based policy optimization methods such as GRPO~\cite{guo2025deepseek} leverage relative comparisons among sampled solutions to improve training stability, making RLVR practical at scale. With outcome-based supervision, RLVR has driven substantial gains across math, coding, logical, and multi-step reasoning tasks, making them as a key component of reasoning-oriented LLM training~\cite{comanici2025gemini, singh2025openai, yang2025qwen3}.
\begin{figure}[t]
    \begin{center}
    \includegraphics[width=0.9\textwidth
    ]{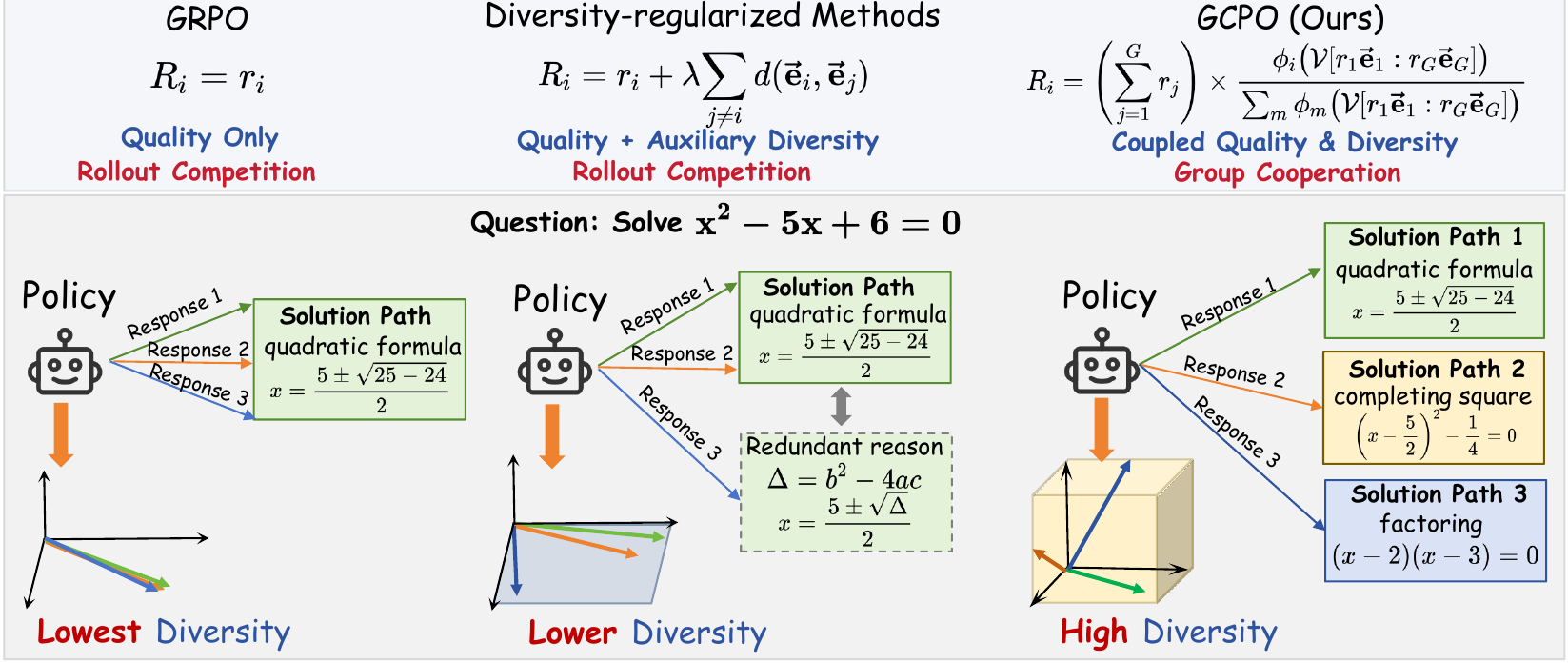}
    \vskip -0.5em
    \caption{
    \textbf{Comparison of group-based optimization strategies.} \textit{Left}: GRPO optimizes solely for individual correctness, often leading to exploration collapse. \textit{Middle}: Diversity-regularized methods add a diversity bonus to the reward, but this only produces superficial variations of already-successful reasoning paths. \textit{Right}: Our GCPO considers each rollout's marginal contribution to the group's coverage of solutions, incentivizing diverse correct reasoning paths.
    }
    \label{fig:intro}
    \end{center}
     \vskip -2em
\end{figure}

Despite these successes, group-based optimization approaches often struggle with exploration collapse: while models improve pass@1 accuracy, they often do so by concentrating probability mass on a narrow set of homogeneous reasoning patterns, lacking the ability to explore new, potentially better reasoning paths~\cite{liu2025prorl,shen2025entropy,wu2025invisible, chen2025eepo}. This highlights a central challenge for RLVR: improving reasoning requires not only reinforcing high-score responses but also sustaining exploration across distinct reasoning paths~\cite{jiang2025risk, walder2025pass}.




Recent works attempt to address this by either adding entropy regularization into the policy distribution~\cite{shen2025entropy,chen2025pass,agarwal2025unreasonable} or enforcing variation among solutions under global constraints~\cite{zhang2025right,chen2025post,li2025jointly,hu2025diversity}. However, these approaches typically treat diversity term as an additional constraint to the correctness rewards, rather than naturally combining them for cooperative optimization. More importantly, such regularization paradigms do not essentially change the \textit{winner-takes-all} nature: rollouts are still evaluated and optimized by their competitive advantages over others. Under this formulation, models tend to conservatively maximize rewards by generating superficial variations of already-successful solutions, rather than exploring truly diverse reasoning paths, as shown in Figure~\ref{fig:intro}. Consequently, the contradiction between quality and diversity remains largely unresolved in LLM reasoning.

In this work, we propose GCPO (\textit{\textbf{G}roup \textbf{C}ooperative \textbf{P}olicy \textbf{O}ptimization}), a novel cooperative algorithm for RL-based LLM reasoning. Instead of scoring each rollout independently, GCPO assigns rewards for them through team-level credit assignment. Specifically, GCPO treats a group of sampled rollouts as a cooperative team, and defines its \textbf{team value} as the team’s coverage of valid reasoning paths. Under this objective, each rollout is rewarded by how much it expands the team’s coverage beyond what the others already provide. GCPO instantiates this team value as a determinantal volume, which couples correctness and reasoning diversity. It increases when correct rollouts cover distinct reasoning paths, while incorrect or redundant ones contribute little additional value. To fairly allocate this team value, GCPO redistributes the team's original reward to each rollout according to its average marginal contribution to the team. This team-level cooperative mechanism preserves the task-level reward budget while routing optimization toward non-redundant correct reasoning paths. Extensive experiments across reasoning benchmarks covering multiple domains show that GCPO improves both reasoning accuracy and meaningful solution diversity.

\section{Related Work}
\label{related_work}

\noindent \textbf{RLVR.} 
RLVR has become a prominent approach for improving LLM reasoning~\cite{cobbe2021training, guo2025deepseek, openai-o1}. While effective, it often suffers from limited exploration, where policies prematurely converge to a narrow set of reasoning patterns~\cite{shen2025entropy,wu2024progress}. To mitigate this, prior work has introduced unstructured randomness via entropy regularization~\cite{hou2025advancing}, modified objectives using pass@k-based rewards~\cite{chen2025pass}, or intervened in rollout dynamics~\cite{chen2025eepo}. However, these methods mainly stimulate exploration by perturbing sampling behavior with exploration-oriented signals. In contrast, GCPO keeps the rollout process simple and introduces cooperation at team-level: sampled rollouts are evaluated as a team, so exploration is guided by their collective coverage rather than isolated stochasticity.

\noindent \textbf{Diversity in LLM Reasoning.} 
Prior studies promote diversity by manipulating policy uncertainty at different levels. Token-level methods primarily rely on entropy bonuses or KL constraints~\cite{cui2025entropy,liu2025ettrl,yu2025dapo,agarwal2025unreasonable,shen2025entropy,yao2025diversity} to increase short-term randomness, but fail to explicitly encourage diverse reasoning rollouts. More recent global approaches leverage pass@k signals~\cite{chen2025pass,walder2025pass} or promote semantic diversity among candidate solutions~\cite{chen2025post,li2025jointly,hu2025diversity}. However, these methods usually introduce diversity as an auxiliary objective, while leaving rollout-level credit assignment unchanged. In contrast, GCPO incorporates diversity into a group-level value function and assigns credit by each rollout's marginal contribution to reward-gated semantic coverage.

\section{GCPO: From Rollout Competition to Team Cooperation}
\label{sec:method}
\begin{figure*}[t]
    \vskip -0.6em
    \begin{center}
    \includegraphics[width=0.97\textwidth,height=0.25\textwidth]{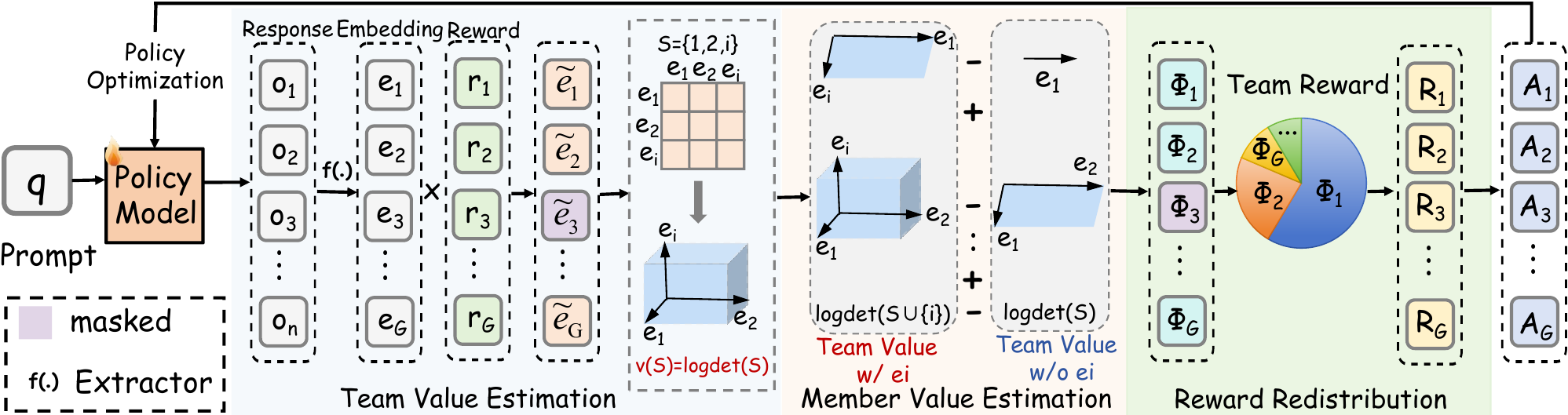}
    \vskip -0.6em
    \caption{
    Pipeline of GCPO framework. We formulate the contribution via coupled diversity and quality of response $o_i$ within a group of $G$ rollouts as reward shaper to incentivize deep exploration.  
    }
    \label{fig:framework}
    \end{center}
    \vskip -1.5em
\end{figure*}


\subsection{Preliminaries} \label{sec:preliminaries}
\textbf{Group Relative Policy Optimization.}
We briefly review GRPO~\cite{guo2025deepseek}, which serves as the optimization backbone for RLVR in reasoning tasks. For each query $q$, the old policy model $\pi_{\theta_{\text{old}}}$ samples $G$ candidate responses $\{o_1, \dots, o_G\}$. Each response is assigned a verifiable binary reward $r_i \in \{0,1\}$ based on matching against the ground-truth answer, which helps mitigate reward hacking~\cite{gao2023scaling}. The policy $\pi_{\theta}$ is updated by maximizing:
\begin{equation}\label{grpo}
J_{\text{GRPO}}(\theta) = \mathbb{E}_{q\sim\mathcal{D}, \{o_i\}_{i=1}^G\sim\pi_{\theta_{\text{old}}}(\cdot|q)} \left[ \frac{1}{G} \sum_{i=1}^{G}  \sum_{t=1}^{|o_i|} \text{CLIP}(\rho_{i,t}, A_i) - \beta \mathbb{KL}(\pi_\theta||\pi_{\text{ref}})  \right],
\end{equation}
where $\rho_{i,t} \!=\! \frac{\pi_\theta(o_{i,t}|q, o_{i,<t})}{\pi_{\theta_{\text{old}}}(o_{i,t}|q, o_{i,<t})}$ is the importance weight. The objective employs PPO-style clipping $\text{CLIP}(\rho_{i,t}, A_i) \!=\! \min(\rho_{i,t} A_i, \text{clip}(\rho_{i,t}, 1-\epsilon, 1+\epsilon)A_i)$ for trust-region updates~\citep{schulman2017proximal}, and a KL penalty to constrain deviation from a reference model $\pi_{\text{ref}}$. 

\vspace{-.22mm}
\textbf{Determinantal Point Processes.}
Determinantal Point Processes (DPPs) \citep{kulesza2012determinantal} are probabilistic models over subsets that favor diversity, with L-ensembles quantifying subset diversity through determinant-based geometric volume. Let $\mathcal{Y} = \{1, 2, \dots, N\}$ be a ground set of items, and $\mathbf{L} \in \mathbb{R}^{N \times N}$ be a symmetric, positive semi-definite matrix. An L-ensemble defines a probability measure over random subsets $Y \subseteq \mathcal{Y}$ such that for every subset $A \subseteq \mathcal{Y}$:
\begin{equation}
    \text{Pr}(Y = A) \propto \det(\mathbf{L}_A)
\end{equation}
where $\mathbf{L}_A$ denotes the principal submatrix of $\mathbf{L}$ indexed by $A$, $\det(\mathbf{L}_A)$ is the determinant of $\mathbf{L}_A$. 

Geometrically, if the entries of $\mathbf{L}$ represent inner products of feature vectors, $\det(\mathbf{L}_A)$ corresponds to the squared volume spanned by the items in $A$. This volume naturally increases when the vectors are diverse and diminishes toward zero when they are similar. 



\subsection{Group Cooperative Policy Optimization}
\label{sec:gcpo_framework}
In this work, we propose GCPO, an algorithm that shifts the \emph{competition} among sampled rollouts to \emph{cooperation}. GCPO treats the rollout group as a cooperative team and rewards each rollout by how much it expands the team's coverage of valid reasoning directions. As illustrated in Figure~\ref{fig:framework}, this is realized through three components: quality-diversity coupled team value estimation (\S\ref{sec:volume}), Shapley-based cooperative member value estimation(\S\ref{sec:credit}), and reward redistribution (\S\ref{sec:redistribution}).

\subsubsection{Team Value Estimation}
\label{sec:volume}

Here, we introduce a team value function to quantify the collective coverage of valid reasoning paths within a rollout group. The proposed formulation couples two desiderata: \textbf{quality}, which measures the objective correctness of individual rollouts, and \textbf{diversity}, which captures the extent to which the team spans distinct and non-redundant semantic directions. 

For each response $o_i$, we extract a semantic representation and weight it by its reward, 
i.e., $z_i = f_\phi(o_i),
\bar z_i = \frac{z_i}{\lVert z_i \rVert_2},
\tilde z_i = r_i \bar z_i$,
where $f_\phi$ denotes the frozen encoder and $r_i$ is the reward assigned to $o_i$. 
We then define a quality-diversity coupled Gram kernel:$L_{ij} = \langle \tilde z_i, \tilde z_j \rangle = r_i r_j \langle \bar z_i, \bar z_j \rangle.$ This kernel induces a team-level semantic geometry over the responses. It measures the quality-weighted overlap between two reasoning rollouts: high-reward responses contribute strongly, whereas similar responses exhibit redundant coverage. Consequently, correct rollouts with distinct semantic directions enlarge the collective coverage of the team, while redundant ones contribute little.


Inspired by the DPP in \S\ref{sec:preliminaries}, we define the team value of a rollout set $S$ as the regularized log-determinant volume 
\begin{equation}
    v(S)=\log\det\left(I_{|S|}+\eta L_S\right).
    \label{eq:team_value}
\end{equation}
$L_S$ is the submatrix of $L$ indexed by $S$ and $\eta>0$ controls its contribution. Identity matrix $I_{|S|}$ provides a positive volume and ensures numerical stability when rollouts are highly correlated.

We adopt a determinant-based formulation for the team value because it measures the collective volume spanned by a response set rather than pairwise separations. This is essential for mitigating the limitations of additive diversity bonuses. Pairwise distances do not reflect team-level diversity, since they may remain large even when responses are organized into internally redundant clusters. In contrast, the determinant captures the spectral structure of the Gram matrix, increasing when responses contribute independent directions and saturating when additional rollouts are redundant with the existing set. Proposition~\ref{prop:spectral_team_value} formalizes these properties.
\begin{proposition} 
\label{prop:spectral_team_value}
Let $\lambda_1,\dots,\lambda_{|S|}$ be the eigenvalues of $L_S$. Then the team value can be expressed as:
\begin{equation}
    v(S)=\sum_{\ell=1}^{|S|}\log(1+\eta\lambda_\ell).
\end{equation}
\end{proposition}
\vspace{-3mm}
This form shows how the value function balances correctness and diversity. Zero-reward rollouts do not contribute to the team value. Moreover, $m$ mutually orthogonal correct responses yield $v(S)=m\log(1+\eta)$, whereas $m$ identical correct responses yield $v(S)=\log(1+\eta m)$. Therefore, the value grows more rapidly when correct responses span independent semantic directions, while it saturates when correct responses are redundant. 
Thus, \textbf{diverse correct reasoning paths are assigned higher team value, whereas redundant correct paths yield diminishing returns.}

\subsubsection{Member Value Estimation}
\label{sec:credit}
The team value measures the collective coverage of a coalition, while policy optimization requires a per-rollout credit that reflects each member’s contribution to the group. This credit is derived in two steps: each rollout’s \textbf{marginal contribution} to a coalition is first computed, and these context-dependent contributions are then aggregated into a fair \textbf{cooperative credit} via Shapley values.

\noindent \textbf{Marginal contribution.} Given a rollout $i$ and coalition $S$ that does not contain $i$, the marginal contribution of $i$ is defined as the additional team value induced by adding $i$ to $S$: 
\begin{equation}
   \Delta_i(S)=v(S\cup\{i\})-v(S).
\label{eq:marginal_definition}
\end{equation}

\begin{theorem}
\label{thm:marginal_contribution}
Let $\tilde Z_S\in\mathbb{R}^{|S|\times d}$ denote the representation matrix of coalition $S$, whose rows are given by $\{\tilde z_j\}_{j\in S}$.
For any $i\notin S$, the marginal contribution can be calculated as:
\begin{equation}
    \Delta_i(S)
    =
    \log
    \left(
    1+
    \eta r_i^2
    \bar{z}_i^\top
    \left(
    I_d+\eta \tilde Z_S^\top \tilde Z_S
    \right)^{-1}
    \bar{z}_i
    \right),
    \label{eq:marginal_quality_novelty}
\end{equation}
where $r_i$ is the reward of i. The marginal contribution is bounded as $0\leq \Delta_i(S)\leq \log(1+\eta r_i^2).$
\end{theorem}

Here, $r_i^2$ gates the contribution and the quadratic form measures the novelty of $z_i$ after accounting for directions already covered by $S$. Theorem~\ref{thm:marginal_contribution} confirms that \textbf{a rollout is valuable only when it is both correct and semantically non-redundant.} Moreover, it is upper-bounded, ensuring that the credit assigned to any rollout remains \textbf{well-controlled during policy optimization.}

\vspace{-1.5mm}
\textbf{Cooperative credits.}
The marginal contribution $\Delta_i(S)$ is coalition-dependent, and thus does not provide a unique rollout-level credit by itself. A straightforward choice is the leave-one-out (LOO) score $\Delta_i([G]\setminus\{i\})$, which fixes $S$ to the rest of the full group.
But LOO suffers from mutual masking: when multiple rollouts instantiate the same valid reasoning direction, removing any single one has only a limited effect on the total volume. As a result, all such rollouts may receive small individual credits, despite collectively contributing a valuable reasoning direction.

To obtain a fair and context-independent credit for each rollout, we use the Shapley value~\cite{sundararajan2020many}, which attributes the team value by averaging each rollout’s marginal contribution across all possible coalition formation orders:
\begin{equation}
    \phi_i
    =
    \sum_{S\subseteq [G]\setminus\{i\}}
    \frac{|S|!(G-|S|-1)!}{G!}
    \left[
    v(S\cup\{i\})-v(S)
    \right].
    \label{eq:gcpo_shapley_all}
\end{equation}
The Shapley value provides a principled way to convert coalition-dependent marginal contributions into a coalition-independent cooperative credit and mitigate the risk of mutual masking.


\begin{theorem}
\label{thm:shapley_redundancy}
Consider a set $C=\{1,\ldots,M\}$ of $m$ verifier-correct rollouts with identical semantic embeddings, i.e., $\tilde z_i=u$ for all $i\in C$. 
Then the Shapley value of each rollout is
\begin{equation}
    \phi_i
    =
    \frac{1}{M}\log\left(1+\eta M\|u\|_2^2\right),
    \qquad i\in C.
    \label{eq:duplicate_shapley}
\end{equation}
\end{theorem}


Theorem~\ref{thm:shapley_redundancy} highlights a key distinction between GCPO and LOO-based credit assignment. LOO measures the decrease in team value after removing a rollout from the full group. As a result, when multiple rollouts realize the same valid reasoning path, the remaining rollouts can compensate for the removed one, causing LOO to underestimate their shared contribution. 
\textbf{Shapley instead divides the team value equally among all rollouts that realize the same reasoning path.}

Since rollouts with zero reward have zero marginal contribution according to Theorem~\ref{thm:marginal_contribution}, they do not affect the Shapley allocation.
Let $\mathcal{P}=\{i\in [G]\mid r_i > 0\}$ be the set of positive-reward rollouts. In practice, we therefore efficiently compute Shapley values only over $\mathcal{P}$:
\begin{equation}
    \phi_i
    =\sum_{S\subseteq \mathcal{P}\setminus\{i\}}
    \frac{|S|!(|\mathcal{P}|-|S|-1)!}{|\mathcal{P}|!}
    \left[
    v(S\cup\{i\})-v(S)
    \right].
    \label{eq:gcpo_shapley_correct}
\end{equation}

\subsubsection{Reward Redistribution}
\label{sec:redistribution}

After computing cooperative values for individual rollouts, GCPO converts them into training signals by reward redistribution. Rather than adding these values as extra bonuses, which would increase the total reward mass and may encourage reward hacking, GCPO keeps the total batch reward fixed and reallocates it to rollouts in proportion to their cooperative credits. Specifically, let $R_{\mathrm{batch}}
    =
    \sum_{i=1}^{G} r_i$, the final reward of GCPO is defined as:
\begin{equation}
    \tilde r_i
    =
    \begin{cases}
    R_{\mathrm{batch}}
    \cdot
    \dfrac{\phi_i}{\sum_{j\in\mathcal{P}}\phi_j},
    & r_i>0, \\[2.2ex]
    0,
    & r_i=0.
    \end{cases}
    \label{eq:gcpo_reward}
\end{equation}
Since $\sum_{i=1}^G \tilde{r}_i = R_{\mathrm{batch}} = \sum_{i=1}^G r_i$, the total reward in the batch remains unchanged. This conservation property ensures that GCPO does not introduce additional reward mass: it preserves the task-level reward budget while reallocating credit among correct rollouts, thereby reducing the risk of reward hacking associated with additive diversity bonuses.

\begin{proposition}
\label{prop:reward_budget}
For a query $q$, let $B=\{o_i\}_{i=1}^G \sim \pi(\cdot\mid q)^G$ be a group of i.i.d. responses with verifier rewards $r_i=r(o_i)$. Then the expected batch-average reward of GCPO is preserved: $J^{\Sigma}_{\mathrm{GCPO}}(\pi;q)
=
\mathbb{E}_{B\sim\pi^G}
\left[
\frac{1}{G}\sum_{i=1}^G \tilde r_i
\right]
=
\mathbb{E}_{o\sim\pi(\cdot\mid q)}[r(o)]
=
J_{\mathrm{obj}}(\pi;q).$
In contrast, for an additive diversity-shaped reward
$r_i^{\mathrm{add}}=r_i+\lambda d_i$, we have: $J_{\mathrm{add}}(\pi;q)
=
J_{\mathrm{obj}}(\pi;q)
+
\lambda\,
\mathbb{E}_{B\sim\pi^G}
\left[
\frac{1}{G}\sum_{i=1}^G d_i
\right].$
\end{proposition}
Proposition~\ref{prop:reward_budget} distinguishes GCPO from additive diversity shaping at the reward-budget level. Additive shaping increases the expected reward by an external diversity term, which may favor diversity even when it is not aligned with correctness. In contrast, GCPO conserves the group-level reward and only reallocates it within the response group. Thus, diversity affects learning only through credit assigned to correct, non-redundant rollouts, rather than through additional reward mass.
\subsubsection{GCPO Objective}
\label{sec:objective}

GCPO can be seamlessly incorporated into GRPO by modifying only the reward used for advantage estimation. Specifically, we replace the original reward  $r_i$ with the redistributed reward $\tilde{r}_i$, leading to $A_i^{GCPO} = (\tilde{r}_i - \mu_{\tilde{r}}) / (\sigma_{\tilde{r}} + \epsilon_{num})$. The policy is optimized using standard clipped surrogate loss:
\begin{equation}
\label{eq:total_loss}
  \mathcal{L}_{GCPO}(\theta) = \mathbb{E}_{x \sim \mathcal{D}} \left[ \frac{1}{G} \sum_{i=1}^G \frac{1}{T_i} \sum_{t=1}^{T_i} \min \left( \rho_{i,t} A_i^{GCPO}, \text{clip}(\rho_{i,t}, 1-\epsilon_c, 1+\epsilon_c) A_i^{GCPO} \right) \right] - \beta \mathbb{D}_{KL} \nonumber
\end{equation}

The encoder $f_\phi$, Shapley credits $\phi_i$, and redistributed rewards $\tilde r_i$ are treated as stop-gradient quantities during policy optimization. Thus, GCPO modifies the training signal only through the advantages induced by reward redistribution. We make this effect explicit through the following decomposition.
\begin{proposition}
\label{prop:gradient_routing}
Let $\Delta A_i \triangleq A_i^{GCPO} - A_i^{GRPO}$ denote the advantage correction introduced by GCPO for the $i$-th rollout. The unclipped policy-gradient surrogate admits the decomposition
\begin{equation}
    \nabla_\theta \mathcal{L}_{GCPO}
    =
    \nabla_\theta \mathcal{L}_{GRPO}
    +
    \mathbb{E}_{x \sim \mathcal{D}}
    \left[
    \frac{1}{G}
    \sum_{i=1}^G
    \nabla_\theta \log \pi_\theta(o_i \mid x)
    \cdot \Delta A_i
    \right].
\end{equation}
\end{proposition}
Proposition~\ref{prop:gradient_routing} reveals GCPO's optimization signal. The first term recovers the standard GRPO reward update, while the second term acts as an internal credit-routing correction induced by reward redistribution. This correction favors semantically complementary correct rollouts over redundant ones, thereby routing exploration within the valid solution space without promoting incorrect rollouts.

\vspace{-1mm}
\textbf{Implementation.}
In practice, GCPO incurs only a modest computational overhead.
Although the coupled kernel is defined over all rollouts, rollouts with zero reward have zero contribution and can be excluded from Shapley computation. Thus, Shapley values can be computed only over $\mathcal{P}$ using Eq.~\eqref{eq:gcpo_shapley_correct}. 
For the rollout group sizes $G\le8$, exact enumeration over correct subsets is feasible. 
For larger groups, we estimate Shapley value by Monte Carlo permutations detailed in Appendix~\ref{app:scalable-shapley}.

\section{Experiments}
\label{sec:experiments}

\vspace{-0.5em}
\subsection{Experimental Setup}

\label{sec:setup}
\vspace{-1.5mm}
\noindent \textbf{Implementation Details.}
We conduct all experiments on the filtered DAPO-Math-17K~\cite{openr1}, which removes duplicated samples, and train base models with increasing capacity: Qwen3-1.7B and 4B~\cite{yang2025qwen3}. We adopt all-MiniLM-L6-v2~\cite{reimers2019sentence} as a lightweight sentence encoder for extracting semantic response embeddings. All methods use the same training configuration: batch size 256, $G=8$ rollouts during policy optimization, learning rate $1e^{-6}$, and maximum response length 8192. 

\vspace{-1.5mm}
\noindent \textbf{Evaluation Setting.}
Evaluation is conducted on a diverse set of benchmarks spanning math reasoning and broader reasoning domains. For in-domain evaluation, we use AIME2024~\cite{aime24}, AIME2025, MATH500~\cite{cobbe2021training}, Minerva Math~\cite{lewkowycz2022solving}, and Olympiad~\cite{he2024olympiadbench}. For out-of-domain generalization, we evaluate on ARC-Challenge~\cite{clark2018think}, GPQA~\cite{rein2023gpqa}, and MMLU-Pro~\cite{wang2024mmlu}, covering multi-domain reasoning. For smaller test sets(AIME 24/25), we report \textbf{Avg@16}, computed by averaging correctness over 16 independent rollouts at temperature 0.7, and \textbf{Pass@1}, computed from a single rollout per problem for others. In addition, we evaluate \textbf{Pass@k} for $k \in \{2, 4, \ldots, 64\}$ by sampling 64 independent rollouts per problem. We compare GCPO against existing sota optimization algorithms like GRPO~\cite{shao2024deepseekmath}, DAPO~\cite{yu2025dapo}, DIVER~\cite{hu2025diversity}, DQO~\cite{chen2025post}, as well as the corresponding base models.

\begin{table*}[t]
\vskip -0.6em
\caption{Comparison on in and out-of-domain tasks. \textbf{Bold}/\underline{underline} denote best/second-best.
}
\vspace{-2mm}
\label{tab:main_results}
\begin{center}
\setlength{\tabcolsep}{2pt}
\renewcommand{\arraystretch}{0.95}
\resizebox{\textwidth}{!}{%
\begin{tabular}{l ccccc >{\columncolor{yellow!12}}c | ccc >{\columncolor{cyan!10}}c}
\toprule
\multirow{2}{*}{\textbf{Methods}}
& \multicolumn{6}{c|}{\textbf{In-Domain Performance}}
& \multicolumn{4}{c}{\textbf{Out-of-Domain Performance}} \\
\cmidrule(lr){2-7} \cmidrule(lr){8-11}
& \textbf{AIME24} & \textbf{AIME25} & \textbf{MATH500} & \textbf{Minerva} & \textbf{Olympiad} & \textbf{Avg.}
& \textbf{ARC-c} & \textbf{GPQA} & \textbf{MMLU-Pro} & \textbf{Avg.} \\
\midrule
Qwen3-1.7B(Base) & 13.3 & 10.0 & 57.6 & 17.3 & 23.7 & 24.4 & 54.1 & 20.2 & 27.5 & 33.9 \\
\midrule
GRPO~\cite{guo2025deepseek}        & 16.7 & 20.0 & 57.4 & 18.4 & 29.5 & 28.4 & 75.6 & 24.3 & 38.5 & 46.1 \\
DAPO~\cite{yu2025dapo}        & 20.0 & 10.0 & \underline{63.8} & 21.3 & 32.9 & 29.6 & \textbf{79.4} & 28.2 & 41.0 & 49.5 \\
Div-R1~\cite{yao2025diversity}  & 22.1 & 15.3 & 60.2 & 19.5 & 31.2 & 29.7 & 77.8 & 25.1 & 39.9 & 47.6 \\
DIVER~\cite{hu2025diversity}       & 25.4 & 17.6 & 61.5 & 20.8 & 33.5 & 31.8 & 78.3 & \underline{30.2} & \underline{43.8} & \underline{50.8} \\
DQO~\cite{chen2025post}          & \underline{27.5} & \underline{18.4} & 62.3 & \underline{22.1} & \underline{34.2} & \underline{32.9} & 76.4 & 28.3 & 41.5 & 48.7 \\
\midrule
\rowcolor{GCPOgreen}
\textbf{GCPO} & \textbf{30.7} & \textbf{21.2} & \textbf{64.7} & \textbf{23.5} & \textbf{36.2} & \textbf{35.3} & \underline{79.1} & \textbf{31.8} & \textbf{45.2} & \textbf{52.0} \\
\midrule
\midrule
Qwen3-4B(Base)   & 13.3 & 16.7 & 64.8 & 26.1 & 32.8 & 30.7 & 66.9 & 26.3 & 30.9 & 41.4 \\
\midrule
GRPO~\cite{guo2025deepseek}        & 36.7 & 33.3 & 64.0 & 25.0 & 37.5 & 39.3 & 84.6 & 44.4 & 60.1 & 63.0 \\
DAPO~\cite{yu2025dapo}        & 33.3 & 40.0 & 59.0 & 25.7 & 36.7 & 38.9 & \underline{89.6} & 46.0 & \underline{62.8} & 66.1 \\
Div-R1~\cite{yao2025diversity} & 38.5 & 42.5 & 61.5 & 26.2 & 37.9 & 41.3 & 86.9 & 45.2 & 61.5 & 64.7 \\
DIVER~\cite{hu2025diversity}       & 42.3 & 44.2 & 62.8 & \underline{26.4} & 38.5 & 42.8 & 88.4 & \textbf{48.5} & 62.0 & \underline{66.3} \\
DQO~\cite{chen2025post}         & \underline{46.8} & \underline{46.5} & \underline{63.7} & 26.0 & \underline{39.2} & \underline{44.4} & 86.5 & 47.2 & 59.2 & 64.3 \\
\midrule
\rowcolor{GCPOgreen}
\textbf{GCPO} & \textbf{51.7} & \textbf{49.0} & \textbf{66.4} & \textbf{26.8} & \textbf{40.4} & \textbf{46.9} & \textbf{91.4} & \underline{47.5} & \textbf{63.8} & \textbf{67.6} \\
\bottomrule
\end{tabular}
}
\end{center}
\vskip -0.8em
\end{table*}

\vskip -0.4em
\subsection{Main Results}
\label{sec:main_results}
\vskip -0.4em
\noindent \textbf{Comparison with RLVR Baselines.}
Table~\ref{tab:main_results} reports results on five in-domain reasoning benchmarks and three out-of-domain reasoning benchmarks across two model scales.
GCPO consistently improves Pass@1 over standard RLVR baselines.
On Qwen3-1.7B, GCPO achieves an in-domain average of 35.3, with large gains on challenging benchmarks such as AIME24 and Olympiad.
The trend remains at the larger scale: on Qwen3-4B, GCPO reaches 46.9, outperforming DAPO by +8.0 points.
These results suggest that GCPO improves accuracy while maintaining broader coverage of reasoning paths.

\noindent \textbf{Comparison with Diversity-Aware Methods.}
GCPO also outperforms diversity-aware baselines across most evaluated settings.
On Qwen3-1.7B, GCPO improves the in-domain average from DQO's 32.9 to 35.3, and raises the OOD average to 52.0. On Qwen3-4B, GCPO further improves the in-domain average over DQO by +2.5 points and achieves the best OOD average of 67.6. These gains suggest that treating diversity as cooperative credit assignment is more effective than additive diversity, especially on harder benchmarks where redundant patterns can limit exploration.

\noindent \textbf{Out-of-Domain Generalization.} As shown in Table~\ref{tab:main_results}, GCPO achieves the highest average OOD accuracy across both model scales. The performance gains are particularly striking on knowledge-intensive benchmarks: On GPQA, GCPO outperforms DAPO by 3.6 points and DIVER by 1.6 points on Qwen3-1.7B. This result demonstrates that the cooperative mechanism effectively mitigates task-specific pattern collapse and cultivates a robust reasoning capability that generalizes well beyond the primary mathematical optimization objective.

\begin{figure*}[t]
    \begin{center}
    \includegraphics[width=0.98\textwidth,height=0.3\textwidth]{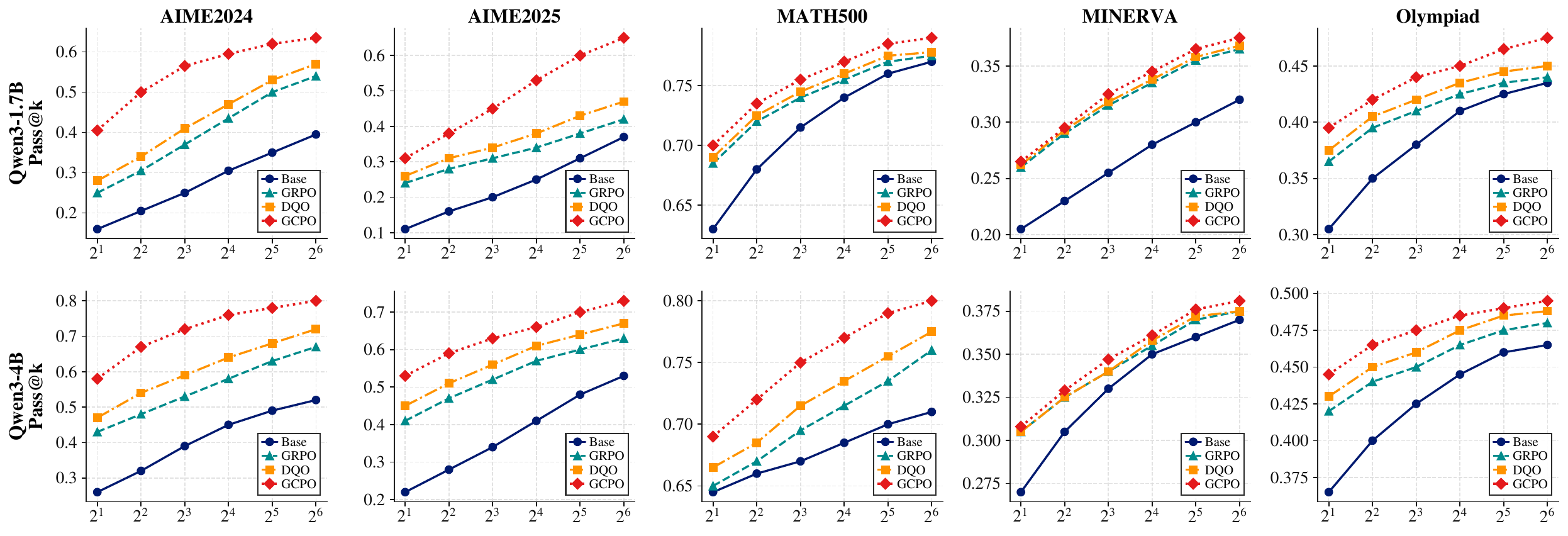}
    \vskip -0.4em
    \caption{
    Pass@k performance across five benchmarks for both Qwen3-1.7B and Qwen3-4B. Base models serve as backbones. GCPO consistently outperforms other models across all values of $k$.
    }
    \label{fig:passk}
    \end{center}
    \vskip -1em
\end{figure*}

\vskip -0.4em
\subsection{Diversity Analysis}
\label{sec:diversity_analysis}
\vskip -0.4em
\vspace{-1mm}
\noindent \textbf{Reasoning Capacity Upper Bound.} 
We evaluate Pass@$k$, the probability of obtaining at least one correct solution from $k$ independent samples, as an empirical measure of the model's solution discovery capacity. 
As shown in Figure~\ref{fig:passk}, GCPO consistently outperforms GRPO, DIVER, and DQO across $k$ at both model scales. 
On challenging benchmarks such as AIME2024 and AIME2025, the performance gap increases with larger $k$, suggesting that GCPO expands the coverage of team solution paths rather than only improving individual samples. 
On MATH500, where baseline accuracy is already high, GCPO still provides stable gains across the full $k$ range, indicating that its benefit is not limited to low-accuracy regimes. 
Overall, these results show that GCPO improves exploratory capacity while maintaining solution quality.

\vspace{-1.5mm}
\noindent \textbf{Multi-dimensional Diversity Characterization.} 
Figure~\ref{fig:diversity}(a) evaluates solution diversity on MATH500 using five metrics: self-ROUGE, self-BLEU, Distinct-1, LLM-as-Judge, and Pass@16, with details provided in Appendix~\ref{app:implementation_details}. 
GCPO shows consistent improvements across all metrics. 
The lexical-level metrics indicate reduced surface redundancy, while LLM-as-Judge and Pass@16 capture semantic reasoning diversity and valid solution coverage. 
Their joint improvement suggests that the determinantal objective encourages meaningful diversity among correct solutions, rather than merely increasing superficial token-level variation.
\begin{figure*}[t]
    \begin{center}
    \includegraphics[width=0.9\textwidth]{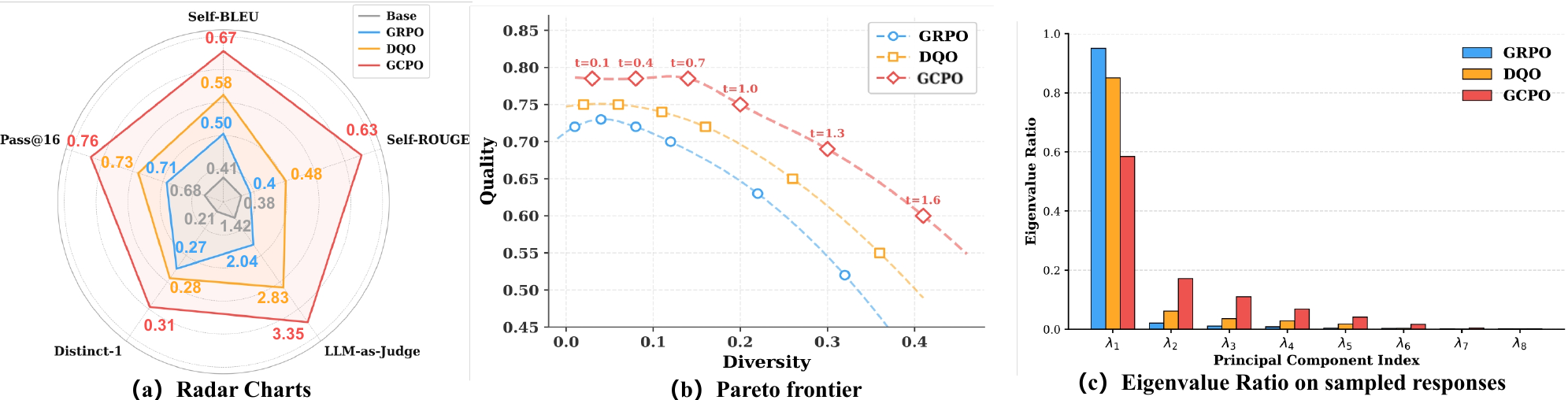}
    \vskip -0.4em
    \caption{
    \textbf{Diversity analysis.} (a) Radar chart comparing five diversity metrics across methods. (b) Pareto frontier at varying sampling temperatures. (c) Eigenvalue Ratio of sampled responses.
    }
    \label{fig:diversity}
    \end{center}
    \vskip -1em
\end{figure*}

\begin{figure*}[t]
    \begin{center}
    \includegraphics[width=0.9\textwidth,height=0.15\textwidth]{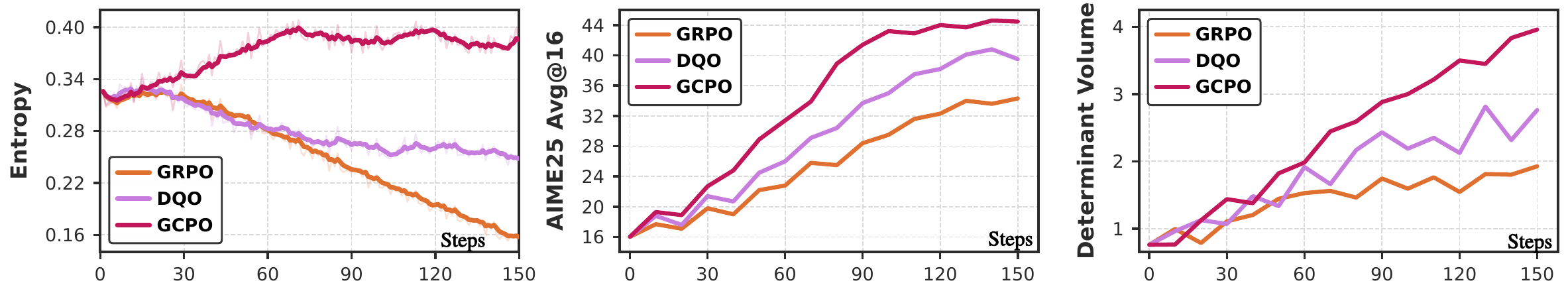}
    \vskip -0.6em
    \caption{
    \textbf{Training dynamics on Qwen3-4B.} From left to right: policy entropy, AIME2025 Avg@16, and the sample rollouts' determinantal team value in Eq.~\eqref{eq:team_value}.
    }
    \label{fig:training_dynamics}
    \end{center}
    \vskip -1.8em
\end{figure*}

\vspace{-1mm}
\noindent \textbf{Quality-Diversity Trade-off.}
Figure~\ref{fig:diversity}(b) studies the trade-off between reasoning quality and solution diversity by varying the sampling temperature $t$, where quality is measured by Pass@1 and diversity by LLM-as-Judge.
As expected, higher temperatures generally increase diversity at the cost of accuracy.
Across this trade-off, GCPO consistently lies closer to the upper-right region than GRPO and DQO, achieving higher accuracy at comparable diversity levels.
This suggests that GCPO encourages diversity that is better aligned with correctness, rather than relying on high-temperature sampling noise.
The result is consistent with our cooperative reward redistribution, which promotes non-redundant valid reasoning paths while preserving the reward budget.

\vspace{-1mm}
\noindent \textbf{Diversity Collapse Analysis.} 
Here, we explore the mechanism behind diversity collapse by eigenvalue analysis. Figure~\ref{fig:diversity}(c) reports the eigenspectrum of the response Gram matrix. 
When sampled responses concentrate on a single reasoning mode, the leading eigenvalue dominates and the remaining spectrum decays rapidly. 
GRPO exhibits this pattern, indicating that its responses are concentrated in a narrow semantic subspace. 
In contrast, GCPO has a lower leading eigenvalue ratio and a flatter spectral profile, suggesting that its responses span more independent semantic directions. This observation is consistent with the team value in Eq.~\eqref{eq:team_value}, which favors correct responses that expand the covered reasoning subspace. To further understand how such diversity collapse emerges during optimization, Figure~\ref{fig:training_dynamics} tracks policy entropy, Avg@16, and determinantal volume over training. 
GRPO and DQO quickly lose exploration capacity: their entropy decreases or plateaus, and their determinantal volume saturates early, indicating that additive bonus mostly reinforce already covered reasoning patterns. 
In contrast, GCPO maintains higher but controlled entropy, while its determinantal volume continues to grow throughout training. 
This suggests that GCPO sustains exploration over non-redundant correct reasoning paths. 
The simultaneous increase in Avg@16 further shows that this expanded coverage translates into improved reasoning performance rather than superficial variation.
\begin{figure}[t]
    \begin{center}
    \includegraphics[width=0.83\textwidth]{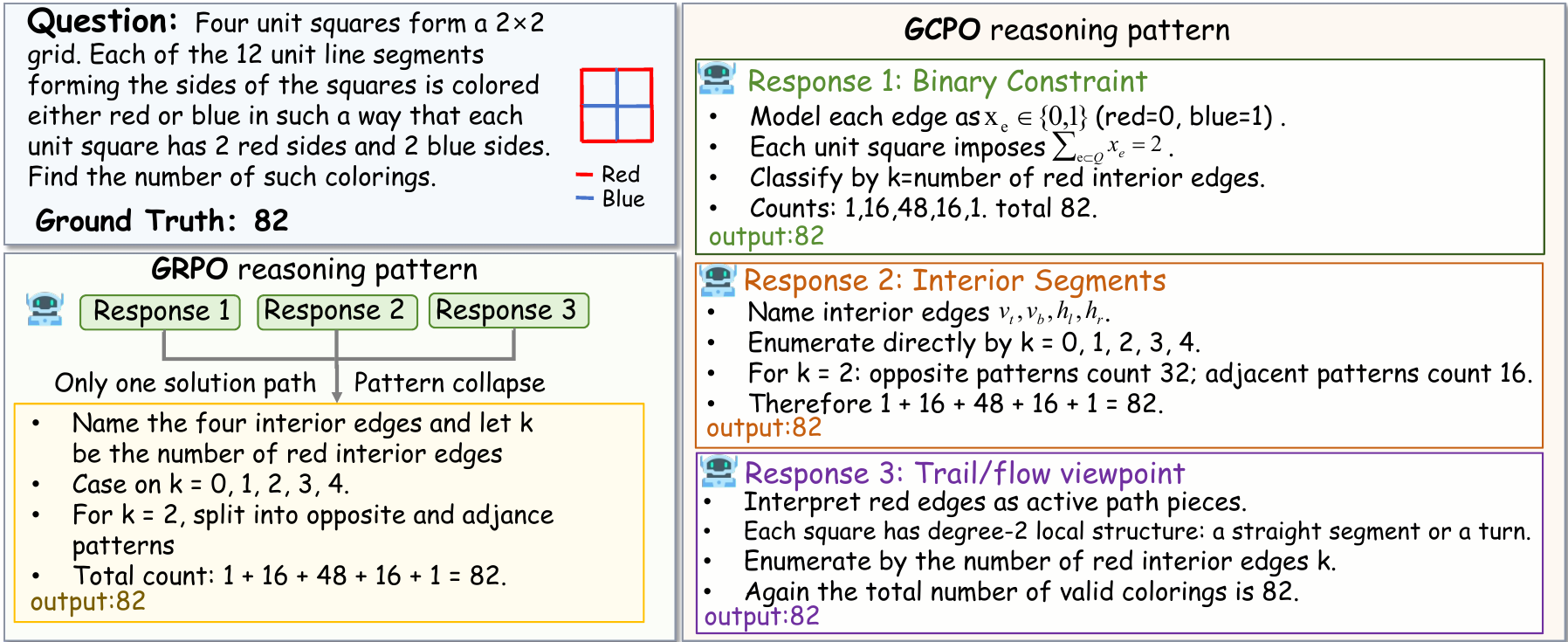}
    \vskip -0.6em
    \caption{
    Qualitative Diversity Analysis of GRPO and GCPO reasoning patterns.
    }
    \label{fig:keshihua}
    \end{center}
    \vskip -1.5em
\end{figure}

\vspace{-1mm}
\noindent \textbf{Qualiative Analysis.} 
Figure~\ref{fig:keshihua} visualizes the reasoning strategies discovered by different methods. Solutions generated by GRPO exhibit pattern collapse, narrowly clustering around standard, low-complexity algebraic solution. In contrast, GCPO expands the solution space along both the complexity and insight axes. It successfully navigates towards high-insight reasoning paths, such as trail/flow viewpoints, which remain unexplored by GRPO, demonstrating that GCPO constructs a richer, multi-modal reasoning model. Please refer to Appendix~\ref{app:case_study} for more qualitative results.

\begin{table}[t]
  \centering
  \caption{Ablation study on the core components of GCPO.  \textbf{Bold} indicates the best performance.
  }
  \label{tab:component_ablation}
  \vspace{-1mm}
  \small
  \setlength{\tabcolsep}{4pt}
  \renewcommand{\arraystretch}{0.9}
  \begin{tabular}{l rr rr}
    \toprule
    \multirow{2}{*}{\textbf{Model Variant}} 
      & \multicolumn{2}{c}{\textbf{Quality}} 
      & \multicolumn{2}{c}{\textbf{Diversity}} \\
    \cmidrule(lr){2-3} \cmidrule(lr){4-5}
      & Pass@1 (\%) & Avg@16 & Pass@16 (\%) & LLM-as-Judge \\ 
    \midrule
    \rowcolor{gray!15}
    \textbf{GCPO} 
      & \textbf{49.02} & \textbf{45.46} & \textbf{66.47} & \textbf{2.58} \\
    \midrule
    w/o Quality-Diversity Coupling 
      & 43.60 & 39.95 & 60.73 & 2.47 \\
    w/o DPP Volume 
      & 46.95 & 43.25 & 62.84 & 1.92 \\
   w/o Shapley Value & 48.23 & 44.50 & 64.05 & 2.17 \\
    Additive Bonus 
      & 48.17 & 44.27 & 64.79 & 2.33 \\
    \bottomrule
  \end{tabular}
  \vskip -1.5em
\end{table}

\vskip -0.4em
\subsection{Ablation Studies}
\vskip -0.4em

Table~\ref{tab:component_ablation} ablates three design hypotheses behind GCPO: quality-diversity coupling, volume-based team coverage, and cooperative reward redistribution.

\vspace{-0.4mm}
\noindent \textbf{Impact of Quality-Diversity Coupling.} 
Removing the quality-diversity coupling causes the largest performance drop, reducing Pass@1/Avg@16 from 49.02/45.46 to 43.60/39.95. 
Notably, the LLM-as-Judge diversity score remains relatively high, indicating that diversity alone is insufficient when it is not tied to correctness. 
This supports our reward-gated formulation, where only positive-reward rollouts can contribute to the team value.

\vspace{-0.4mm}
\noindent \textbf{Impact of Determinantal Team Volume.} 
Removing the DPP-based team volume mainly weakens diversity-related metrics, decreasing LLM-as-Judge from 2.58 to 1.92 and Pass@16 from 66.47 to 62.84. 
This suggests that local or pairwise diversity signals are less effective at modeling group-level coverage. 
The log-determinant volume provides a stronger structural signal by assigning diminishing value to redundant rollouts and favoring responses that expand independent reasoning directions.

\vspace{-0.4mm}
\noindent \textbf{Impact of Cooperative Credit and Reward Redistribution.} 
Replacing Shapley credit with a leave-one-out estimator degrades both accuracy and diversity, showing the importance of fair credit allocation among rollouts that share similar reasoning paths. This is consistent with our analysis that LOO can underestimate the shared value of duplicated discoveries, whereas Shapley distributes the team value across all contributors. 
In addition, the additive-bonus variant underperforms GCPO, indicating that simply adding diversity reward is less effective than redistributing the original verifier reward. Together, these results show that GCPO's gains come from the combination of correctness-gated coverage, determinantal team value, and cooperative reward redistribution.
\begin{wraptable}{r}{0.46\linewidth}
    \setlength{\tabcolsep}{6pt}
    \centering
    \vspace{-1.0em}
    \caption{Computation overhead analysis.}
    \label{tab:transposed_performance}
    \renewcommand{\arraystretch}{0.88}
    \small
    \begin{tabular}{lcc}
        \toprule
        \textbf{Method} & \textbf{Time/step(s)} & \textbf{Acc.(\%)} \\
        \midrule
        GRPO & 316.9 & 33.3 \\
        \midrule
        DIVER & 320.5 {\scriptsize(\textcolor{red}{+3.6})} 
              & 44.2 {\scriptsize(\textcolor{darkgreen}{+10.9})} \\
        GCPO & 320.6 {\scriptsize(\textcolor{red}{+3.7})} 
             & 49.0 {\scriptsize(\textcolor{darkgreen}{+15.7})} \\
        \bottomrule
    \end{tabular}
    \vspace{-1.0em}
\end{wraptable}
\noindent \textbf{Computational Overhead Analysis.} Table~\ref{tab:transposed_performance} reports the per-step runtime of different optimization methods. 
GCPO requires 320.6s per step, introducing only a 3.7s overhead over GRPO and a runtime comparable to DIVER~\cite{hu2025diversity}. 
Meanwhile, GCPO achieves a Pass@1 accuracy of 49.0\% on AIME25, improving over GRPO by 15.7 points. 
These results indicate that our cooperative mechanism introduces limited computational overhead while providing substantial gains in reasoning performance.


\vspace{-1em}
\section{Conclusion}
\label{conclusion}
\vspace{-0.5em}
We introduced GCPO, a cooperative policy optimization framework for improving diverse exploration and mitigating reasoning-pattern collapse in RLVR. GCPO evaluates each rollout by its contribution to the group's coverage of valid reasoning paths, rather than by standalone quality or diversity. It operationalizes this through a reward-weighted determinantal team value, Shapley-based cooperative member value estimation, and redistribution of the original reward. This design preserves the task-level reward budget while assigning preference signals to non-redundant correct rollouts. Experiments across benchmarks and model scales show that GCPO improves both reasoning accuracy and solution diversity. 
A key limitation lies in the cost of determinant evaluation and Shapley enumeration for larger groups, which motivates future work on scalable approximations and extensions to broader reasoning and open-ended generation tasks.

\bibliographystyle{plain}
\bibliography{reference}







\appendix
\onecolumn

\section{Implementation Details}
\label{app:implementation_details}
\subsection{Detailed Settings}

We provide additional experimental details in Section~\ref{sec:experiments}.
All models are trained using the \textsc{verl} framework~\cite{sheng2025hybridflow} and deployed on 8$\times$ NVIDIA 5880 Ada Generation GPUs. Table~\ref{tab:training_details} and Table~\ref{tab:evaluation_settings} summarize the training and evaluation hyperparameters.
Unless otherwise specified, we adopt a rollout group size of $n=8$, a learning rate of $1\times10^{-6}$, and binary verifier rewards. 

\begin{table*}[h]
\begin{center}
\caption{Summary of training details.}
\label{tab:training_details}
\begin{tabular*}{\textwidth}{@{\extracolsep{\fill}} l l}
\toprule
\multicolumn{2}{l}{\textbf{Training Settings}} \\
\midrule
Hardware & 8$\times$5880 Ada Generation GPUs \\
Base models &  Qwen3-1.7B, 4B \\
Training dataset & open-r1/DAPO-Math-17K-Processed \\
Max prompt length & 1024 \\
Max response length & 4096 / 8192 \\
Gpu memory utilization & 0.8 \\
Train batch / mini-batch size & 256 / 16 \\
Use dynamic batch size & True\\
Validate batch size & 512\\
Rollout group size $n$ & 8 \\
Learning rate & $1 \times 10^{-6}$ \\
Temperature (training) & 1.0 \\
Clip range $(\epsilon_{\text{low}}, \epsilon_{\text{high}})$ & (0.2, 0.28) / (0.2, 0.2) for GRPO \\
Reward type & Binary reward \\
Team value factor $\eta$& 1.0 \\
\bottomrule
\end{tabular*}
\end{center}
\vskip -0.1in
\end{table*}

\begin{table*}[h]
\begin{center}
\caption{Evaluation settings.}
\label{tab:evaluation_settings}
\begin{tabular*}{\textwidth}{@{\extracolsep{\fill}} l l}
\toprule
\multicolumn{2}{l}{\textbf{Evaluation settings}} \\
\midrule
Max response length & 4096 / 8192 \\
Top-$p$ (eval) & 0.9 \\
Temperature (eval) & 0.1 for Pass@1; 0.7 for Avg@16 \\
\bottomrule
\end{tabular*}
\end{center}
\vskip -0.1in
\end{table*}

\subsection{Detailed Algorithm}
Here, we provide detailed algorithm for Section~\ref{sec:gcpo_framework}. As shown in Alg~\ref{alg:gcpo_function}, this cooperative framework ensures that rollouts contributing unique semantic directions to the team receive higher credit, while redundant responses are penalized, all while preserving the total reward budget.

\begin{algorithm}[t]
\renewcommand{\algorithmicrequire}{\textbf{Input:}}
\renewcommand{\algorithmicensure}{\textbf{Output:}}
\caption{GCPO reward construction for one rollout group}
\label{alg:gcpo_function}
\begin{algorithmic}[1]
\Require Responses $\{o_i\}_{i=1}^G$, verifier rewards $\{r_i\}_{i=1}^G$, encoder $f_\phi$, scale $\eta$, numerical constant $\epsilon_{\mathrm{num}}$
\Ensure Redistributed rewards $\{\tilde r_i\}_{i=1}^G$

\State $\mathcal{P} \gets \{i\in [G]\mid r_i>0\}$, \quad $R_{\mathrm{batch}}\gets \sum_{i=1}^G r_i$
\If{$\mathcal{P}=\emptyset$}
    \State Set $\tilde r_i=0$ and $A_i^{\mathrm{GCPO}}=0$ for all $i$, then stop
\EndIf

\Statex \textbf{Stage 1: Calculate Team Value}
\For{$i=1$ to $G$}
    \State $z_i \gets f_\phi(o_i)$
    \State $\bar z_i \gets z_i / (\lVert z_i\rVert_2+\epsilon_{\mathrm{num}})$
    \State $\tilde z_i \gets r_i\,\bar z_i$
\EndFor
\State $L \gets [\langle \tilde z_i,\tilde z_j\rangle]_{i,j=1}^G$
\State Define $v(\emptyset)\gets 0$ and $v(S)\gets \log\det(I_{|S|}+\eta L_S)$ for any non-empty $S\subseteq\mathcal{P}$

\Statex \textbf{Stage 2: Assign Member Value}
\State $m\gets |\mathcal{P}|$
\For{$i=1$ to $G$}
    \State $\phi_i\gets 0$
\EndFor
\For{each $i\in\mathcal{P}$}
    \For{each subset $S\subseteq\mathcal{P}\setminus\{i\}$}
        \State $w(S)\gets \dfrac{|S|!(m-|S|-1)!}{m!}$
        \State $\phi_i\gets \phi_i+w(S)\bigl[v(S\cup\{i\})-v(S)\bigr]$
    \EndFor
\EndFor

\Statex \textbf{Stage 3: Redistribute reward}
\State $\Phi\gets \sum_{j\in\mathcal{P}}\phi_j$
\For{$i=1$ to $G$}
    \If{$i\in\mathcal{P}$ and $\Phi>0$}
        \State $\tilde r_i\gets R_{\mathrm{batch}}\cdot \phi_i/\Phi$
    \Else
        \State $\tilde r_i\gets 0$
    \EndIf
\EndFor
\end{algorithmic}
\end{algorithm}

\subsection{Shapley Estimation for Large Rollout Groups}
\label{app:scalable-shapley}

Exact Shapley computation requires summing over exponentially many coalitions. 
In our main experiments, this cost is small because the rollout group size is moderate and Shapley values are computed only over verifier-correct rollouts. 
For larger rollout groups, we use an unbiased permutation estimator with incremental log-determinant updates.
Let $P=\{i\in [G]\mid r_i>0\}$ be the set of positive-reward rollouts. 
By Theorem~\ref{thm:marginal_contribution}, any rollout with $r_i=0$ has zero marginal contribution to every coalition. Hence Shapley computation can be restricted to $P$:
\[
\phi_i
=
\sum_{S\subseteq P\setminus\{i\}}
\frac{|S|!(m-|S|-1)!}{m!}
\left[
v(S\cup\{i\})-v(S)
\right],
\qquad i\in P .
\]
This is still exponential in $m$. 
We now describe a scalable estimator for large $m$.

\paragraph{Permutation estimator.}
Let $\Pi(P)$ denote the set of all permutations of $P$. 
For a permutation $\pi\in\Pi(P)$, let $S_i^\pi$ be the set of rollouts that appear before $i$ in $\pi$. 
The Shapley value admits the equivalent form
\[
\phi_i
=
\mathbb{E}_{\pi\sim \mathrm{Unif}(\Pi(P))}
\left[
v(S_i^\pi\cup\{i\})-v(S_i^\pi)
\right].
\]
To see this, fix any coalition $S\subseteq P\setminus\{i\}$. 
The event $S_i^\pi=S$ occurs when the elements of $S$ appear before $i$, the elements of $P\setminus(S\cup\{i\})$ appear after $i$, and the internal orders within the two groups are arbitrary. 
Thus,
\[
\Pr(S_i^\pi=S)
=
\frac{|S|!(m-|S|-1)!}{m!},
\]
which is exactly the Shapley coefficient.

We sample $K$ independent permutations $\{\pi_k\}_{k=1}^K$ and estimate $\widehat{\phi}_i$ with
$$
\widehat{\phi}_i
=
\frac{1}{K}
\sum_{k=1}^{K}
\left[
v(S_i^{\pi_k}\cup\{i\})-v(S_i^{\pi_k})
\right].
$$ 
This estimator is unbiased: $\mathbb{E}[\widehat{\phi}_i]=\phi_i .$

\paragraph{Incremental log-determinant update.}
A direct implementation would recompute a log-determinant for every marginal contribution. 
Instead, we compute all marginal contributions along a sampled permutation incrementally.

For a coalition $S$, define
\[
\mathbf K_S=\mathbf I_{|S|}+\eta\mathbf L_S .
\]
Suppose we add a rollout $i\notin S$. 
After ordering the indices as $(S,i)$,
\[
\mathbf K_{S\cup\{i\}}
=
\begin{bmatrix}
\mathbf K_S & \eta\mathbf L_{S,i}\\
\eta\mathbf L_{i,S} & 1+\eta L_{ii}
\end{bmatrix}.
\]
By the Schur complement,
\[
\det(\mathbf K_{S\cup\{i\}})
=
\det(\mathbf K_S)
\left(
1+\eta L_{ii}
-
\eta^2\mathbf L_{i,S}\mathbf K_S^{-1}\mathbf L_{S,i}
\right).
\]
Therefore,
\[
v(S\cup\{i\})-v(S)
=
\log
\left(
1+\eta L_{ii}
-
\eta^2\mathbf L_{i,S}\mathbf K_S^{-1}\mathbf L_{S,i}
\right).
\]

In practice, we maintain the Cholesky factor
\[
\mathbf K_S=\mathbf C_S\mathbf C_S^\top .
\]
Let $\mathbf b_{S,i}=\eta\mathbf L_{S,i},
a_i=1+\eta L_{ii}.$
Solving $\mathbf C_S\mathbf y=\mathbf b_{S,i}$
gives:
\[
v(S\cup\{i\})-v(S)
=
\log s_i,
\qquad
s_i=a_i-\|\mathbf y\|_2^2 .
\]
The Cholesky factor is then updated as
\[
\mathbf C_{S\cup\{i\}}
=
\begin{bmatrix}
\mathbf C_S & \mathbf 0\\
\mathbf y^\top & \sqrt{s_i}
\end{bmatrix}.
\]
For the empty coalition, this reduces to
\[
v(\{i\})-v(\emptyset)=\log(1+\eta L_{ii}).
\]
Since $\mathbf L\succeq 0$ and $\eta>0$, $\mathbf K_S$ is positive definite and the Schur complement satisfies $s_i>0$.

\paragraph{Complexity.}
Exact computation over $m$ positive-reward rollouts requires evaluating all $2^m$ coalitions. 
With direct log-determinant evaluation, this costs $O(2^m m^3)$ time in the worst case, plus $O(m2^m)$ time for Shapley aggregation.

With the permutation estimator, each sampled permutation processes the $m$ rollouts once. 
At step $t$, the triangular solve and Cholesky update cost $O(t^2)$, so one full permutation costs
\[
\sum_{t=1}^{m}O(t^2)=O(m^3).
\]
The total complexity is therefore $O(Km^3)$ time and $O(m^2)$ memory. 
When $K\ll 2^m$, this provides a substantial reduction in computation while preserving an unbiased estimate of the Shapley credits.

\paragraph{Concentration.}
The estimator also has a high-probability error bound. 
Under binary verifier rewards, for any $i\in P$, Theorem~\ref{thm:marginal_contribution} gives
\[
0\leq v(S\cup\{i\})-v(S)\leq b,
\qquad
b=\log(1+\eta).
\]
Thus $\widehat{\phi}_i$ is the average of $K$ independent bounded random variables. 
By Hoeffding's inequality,
\[
\Pr\left(
|\widehat{\phi}_i-\phi_i|\geq \epsilon
\right)
\leq
2\exp\left(
-\frac{2K\epsilon^2}{b^2}
\right).
\]
Applying a union bound over all $i\in P$ yields
\[
\Pr\left(
\max_{i\in P}
|\widehat{\phi}_i-\phi_i|
\geq \epsilon
\right)
\leq
2m\exp\left(
-\frac{2K\epsilon^2}{b^2}
\right).
\]
Equivalently, choosing
\[
K
\geq
\frac{b^2}{2\epsilon^2}
\log\frac{2m}{\delta}
\]
ensures
\[
\max_{i\in P}
|\widehat{\phi}_i-\phi_i|
\leq \epsilon
\]
with probability at least $1-\delta$.

\paragraph{Reward conservation.}
Although the estimator approximates individual credits, it preserves the team value exactly at the level of each sampled permutation. 
For any permutation $\pi$,
\[
\sum_{i\in P}
\left[
v(S_i^\pi\cup\{i\})-v(S_i^\pi)
\right]
=
v(P)-v(\emptyset)
=
v(P),
\]
because the marginal contributions telescope along the permutation. 
Averaging over the $K$ sampled permutations gives $\sum_{i\in P}\widehat{\phi}_i=v(P).$ Therefore, using $\widehat{\phi}_i$ in reward redistribution preserves the original batch reward:
\[
\widetilde r_i
=
R_{\mathrm{batch}}
\frac{\widehat{\phi}_i}
{\sum_{j\in P}\widehat{\phi}_j},
\qquad i\in P,
\]
and $\widetilde r_i=0$ for $i\notin P$. 
Hence, $\sum_{i=1}^{G}\widetilde r_i
=
R_{\mathrm{batch}},$ which means the Monte Carlo estimator does not introduce additional reward mass.

\section{Limitations}
Although GCPO improves both reasoning accuracy and solution diversity in our experiments, several limitations remain. First, GCPO requires additional computation to estimate group-level coverage and cooperative credits. In our experiments, this overhead is limited because the rollout group size is moderate and only verifier-correct responses are used for credit assignment. The runtime analysis shows that GCPO remains comparable to existing diversity-aware RLVR baselines while providing stronger empirical gains. Future work will explore approximate credit estimation and more efficient representation strategies to further reduce the training overhead. 
Second, GCPO relies on semantic embeddings to estimate the diversity and redundancy of reasoning paths. The determinantal team value assumes that distances or angles in the embedding space can meaningfully reflect differences in reasoning strategies. While our ablation studies suggest that GCPO is relatively robust across several sentence encoders, embedding-based similarity remains an imperfect proxy for reasoning-level diversity. Future work may incorporate task-adaptive encoders, structured reasoning representations, or verifier-aware similarity measures to better capture meaningful diversity. A more principled and adaptive diversity measure could better capture underlying semantic variation and automatically adjust to task-specific requirements, as different reasoning tasks may emphasize different aspects of diversity.

\section{Broader Impacts}
Our cooperative policy optimization framework for LLM reasoning offers potential benefits for applications that require diverse and reliable problem solving. 
By encouraging multiple non-redundant valid reasoning paths, the method can support AI tutoring systems with alternative explanations and solution strategies, and may assist scientific research by facilitating hypothesis generation and exploratory problem solving. 
It may also improve the robustness of reasoning-oriented LLMs by reducing dependence on a narrow set of repeated solution patterns. 
However, diverse reasoning does not by itself guarantee factual correctness or safety. 
When applied to high-stakes domains, such systems should be used with human oversight, domain-specific validation, and appropriate safeguards against biased or incorrect reasoning.

\section{Detailed Metrics Definition}
\label{app:metrics_definition}

This section provides formal definitions for the evaluation metrics used throughout the paper. We cover both quality metrics (Pass@k) and diversity metrics (self-BLEU, self-ROUGE, Distinct-n, and LLM-as-Judge).

\subsection{Pass@k}
\label{app:passk_definition}

The pass@k metric evaluates the probability of obtaining at least one correct response among $k$ independently sampled solutions. While pass@1 measures single-attempt accuracy, pass@k emphasizes the model's ability to generate diverse correct solutions and improve success rates through sampling. Given a question $x$ and $k$ i.i.d. responses with binary rewards $r_i \in \{0,1\}$, the pass@k metric is defined as:
\begin{equation}
    \text{pass@k} = \mathbb{P}\left[\bigvee_{i=1}^{k} (r_i = 1)\right] = \mathbb{E}\left[1 - \prod_{i=1}^{k}(1 - r_i)\right].
\end{equation}
In practice, when we sample exactly $n \geq k$ solutions per problem, we estimate pass@k using the unbiased estimator:
\begin{equation}
    \text{pass@k} \approx 1 - \frac{\binom{n-c}{k}}{\binom{n}{k}},
\end{equation}
where $c$ is the number of correct solutions among the $n$ samples. This formulation accounts for sampling without replacement and provides a consistent estimate as $n$ increases.

\subsection{Self-BLEU}
\label{app:selfbleu_definition}

Self-BLEU measures the average n-gram overlap among generated responses, quantifying lexical redundancy. Given a set of $N$ responses $\{o_1, \ldots, o_N\}$ for a single prompt, we compute the BLEU score of each response $o_i$ against all other responses $\{o_j\}_{j \neq i}$ as references. The self-BLEU score is then:
\begin{equation}
    \text{self-BLEU} = \frac{1}{N} \sum_{i=1}^{N} \text{BLEU}(o_i, \{o_j\}_{j \neq i}).
\end{equation}
Higher self-BLEU indicates greater lexical similarity (lower diversity). In our experiments, we report \textbf{diversity} as $1 - \text{self-BLEU}$, so that higher values indicate more diverse outputs. We use BLEU-4 with smoothing for all experiments.

\subsection{Self-ROUGE}
\label{app:selfrouge_definition}

Self-ROUGE measures the average longest common subsequence (LCS) overlap among generated responses. Similar to self-BLEU, for a set of $N$ responses $\{o_1, \ldots, o_N\}$, we compute:
\begin{equation}
    \text{self-ROUGE-L} = \frac{1}{N} \sum_{i=1}^{N} \text{ROUGE-L}(o_i, \{o_j\}_{j \neq i}),
\end{equation}
where ROUGE-L measures the F-score of the longest common subsequence between $o_i$ and the reference set. Higher self-ROUGE indicates greater structural similarity (lower diversity). We report \textbf{diversity} as $1 - \text{self-ROUGE-L}$, so that higher values indicate more diverse reasoning structures.

\subsection{Distinct-n}
\label{app:distinctn_definition}

Distinct-n measures the ratio of unique n-grams to total n-grams in the generated text, capturing token-level diversity. For a set of responses $\{o_1, \ldots, o_N\}$, let $\mathcal{G}_n$ denote the multiset of all n-grams extracted from these responses. Then:
\begin{equation}
    \text{Distinct-n} = \frac{|\text{unique}(\mathcal{G}_n)|}{|\mathcal{G}_n|},
\end{equation}
where $|\text{unique}(\mathcal{G}_n)|$ is the number of distinct n-grams and $|\mathcal{G}_n|$ is the total count. Higher Distinct-n indicates greater token-level variety. In our experiments, we primarily report Distinct-1, which captures vocabulary richness across generated solutions.

\subsection{LLM-as-Judge Diversity}
\label{app:llmjudge_definition}

LLM-as-Judge diversity uses a large language model to evaluate the semantic and reasoning-level diversity of generated solutions. We use GPT5.4~\cite{singh2025openai} to evaluate the diversity. Given a problem and a set of $N$ responses, we prompt a judge model to score diversity across three dimensions: logic diversity (differences in reasoning flow), formula diversity (differences in mathematical techniques), and semantic diversity (differences in overall approach). Each dimension is scored on a 0--10 scale, and the overall diversity score is the average of the three. The prompt template is shown below.

\begin{tcolorbox}[
    colback=gray!4,
    colframe=gray!55!black,
    width=\textwidth,
    enlarge left by=-\parindent,
    enlarge right by=-\parindent,
    arc=4pt,
    boxrule=0.8pt,
    title=\textbf{Prompt Template for LLM-as-Judge Diversity Scoring},
    fonttitle=\bfseries,
]

\textbf{System:}
You are a strict but fair grader. You will be given one math problem and
\textbf{32 independent rollouts (solutions)}.
Your task is to score diversity across rollouts.

\vskip 0.5em
Evaluate the following three dimensions on a \textbf{0--10 scale}:

\begin{itemize}
    \item \textbf{logic\_diversity}: differences in reasoning flow or step order
    \item \textbf{formula\_diversity}: differences in formulas, identities, or techniques used
    \item \textbf{semantic\_diversity}: differences in overall semantic approach or framing
\end{itemize}

\vskip 0.5em
Then output \textbf{overall diversity} as the average of the three scores
(\textbf{0--10}, rounded to \textbf{one decimal}).

\vskip 0.5em
Be objective and do not prefer any model.
\textbf{Output JSON only.}

\end{tcolorbox}

This metric captures high-level reasoning diversity that lexical metrics (self-BLEU, self-ROUGE, Distinct-n) may miss, providing a complementary view of solution variety.

\section{Additional Results}
\label{app:add_results}

\subsection{Pass@k Performance}
\label{app:passk_performance}
To comprehensively demonstrate GCPO's consistent advantage across different sampling budgets, we evaluate pass@k performance for $k \in \{2, 4, 8, 16, 32, 64\}$ on five mathematical reasoning benchmarks: AIME2024, AIME2025, MATH500, MINERVA, and Olympiad. We compare GCPO against the base model, GRPO, and DQO across two model scales: Qwen3-1.7B and Qwen3-4B.

As shown in Tables~\ref{tab:passk_1.7b} and~\ref{tab:passk_4b}, GCPO consistently outperforms all baseline methods across the entire spectrum of pass@k metrics. The performance gap becomes particularly pronounced at higher $k$ values, where GCPO's superior exploration capabilities enable it to discover more diverse correct reasoning paths. This trend holds consistently across all benchmarks and both model scales, confirming that GCPO's cooperative credit assignment mechanism effectively expands the coverage of valid reasoning strategies without sacrificing correctness.

\begin{table*}[t]
\centering
\caption{Pass@k performance (\%) for Qwen3-1.7B across different benchmarks.}
\label{tab:passk_1.7b}
\setlength{\tabcolsep}{3pt}
\renewcommand{\arraystretch}{0.5}
\resizebox{\textwidth}{!}{%
\begin{tabular}{cccccccc}
\toprule
\textbf{Pass@k} & \textbf{Method} & \textbf{AIME2024} & \textbf{AIME2025} & \textbf{MATH500} & \textbf{MINERVA} & \textbf{Olympiad} & \textbf{Avg.} \\
\midrule

\multirow{4}{*}{\makecell{$k=2$}}
&Base & 16.3 & 11.4 & 62.5 & 20.1 & 30.8 & 28.2 \\
&GRPO & 24.6 & 23.9 & 68.2 & 25.7 & 36.0 & 35.7 \\
&DQO & 28.1 & 26.3 & 69.4 & 26.6 & 37.5 & 37.6 \\
\cmidrule(lr){2-8}
&\textbf{GCPO} & \textbf{40.7} & \textbf{30.8} & \textbf{70.3} & \textbf{26.9} & \textbf{39.2} & \textbf{41.6} \\
\midrule

\multirow{4}{*}{\makecell{$k=4$}}
&Base & 20.9 & 15.7 & 67.6 & 22.5 & 35.1 & 32.4 \\
&GRPO & 30.2 & 28.3 & 71.8 & 28.6 & 39.4 & 39.7 \\
&DQO & 33.5 & 31.1 & 72.7 & 29.3 & 40.8 & 41.5 \\
\cmidrule(lr){2-8}
&\textbf{GCPO} & \textbf{49.4} & \textbf{38.2} & \textbf{73.5} & \textbf{30.0} & \textbf{41.6} & \textbf{46.5} \\
\midrule

\multirow{4}{*}{\makecell{$k=8$}}
&Base & 25.2 & 20.5 & 71.1 & 25.4 & 38.3 & 36.1 \\
&GRPO & 37.4 & 30.8 & 74.0 & 31.5 & 41.1 & 43.0 \\
&DQO & 40.6 & 34.2 & 74.8 & 31.9 & 42.4 & 44.8 \\
\cmidrule(lr){2-8}
&\textbf{GCPO} & \textbf{56.3} & \textbf{44.7} & \textbf{75.6} & \textbf{32.8} & \textbf{43.9} & \textbf{50.7} \\
\midrule

\multirow{4}{*}{\makecell{$k=16$}}
&Base & 30.0 & 24.8 & 73.9 & 27.6 & 41.0 & 39.5 \\
&GRPO & 43.7 & 34.0 & 75.4 & 33.2 & 42.8 & 45.8 \\
&DQO & 47.2 & 37.7 & 76.3 & 34.1 & 43.7 & 47.8 \\
\cmidrule(lr){2-8}
&\textbf{GCPO} & \textbf{59.1} & \textbf{52.5} & \textbf{77.0} & \textbf{34.6} & \textbf{45.2} & \textbf{53.7} \\
\midrule

\multirow{4}{*}{\makecell{$k=32$}}
&Base & 34.5 & 31.4 & 75.8 & 29.7 & 42.3 & 42.7 \\
&GRPO & 49.6 & 38.5 & 77.3 & 35.1 & 43.6 & 48.8 \\
&DQO & 52.8 & 42.9 & 78.0 & 35.8 & 44.1 & 50.7 \\
\cmidrule(lr){2-8}
&\textbf{GCPO} & \textbf{62.2} & \textbf{59.8} & \textbf{78.4} & \textbf{36.4} & \textbf{46.5} & \textbf{56.7} \\
\midrule

\multirow{4}{*}{\makecell{$k=64$}}
&Base & 39.1 & 37.5 & 76.9 & 32.0 & 43.8 & 45.9 \\
&GRPO & 54.2 & 42.4 & 77.6 & 36.6 & 44.3 & 51.0 \\
&DQO & 57.3 & 46.8 & 78.2 & 37.0 & 45.4 & 52.9 \\
\cmidrule(lr){2-8}
&\textbf{GCPO} & \textbf{63.0} & \textbf{64.7} & \textbf{79.3} & \textbf{37.5} & \textbf{47.1} & \textbf{58.3} \\
\bottomrule
\end{tabular}
}
\end{table*}

\begin{table*}[ht]
\centering
\caption{Pass@k performance (\%) for Qwen3-4B across different benchmarks.}
\label{tab:passk_4b}
\setlength{\tabcolsep}{3pt}
\renewcommand{\arraystretch}{0.5}
\resizebox{\textwidth}{!}{%
\begin{tabular}{cccccccc}
\toprule
\textbf{Pass@k} & \textbf{Method} & \textbf{AIME2024} & \textbf{AIME2025} & \textbf{MATH500} & \textbf{MINERVA} & \textbf{Olympiad} & \textbf{Avg.} \\
\midrule

\multirow{4}{*}{\makecell{$k=2$}}
&Base & 26.1 & 21.8 & 64.3 & 27.2 & 36.4 & 35.2 \\
&GRPO & 42.8 & 41.2 & 65.1 & 30.4 & 41.9 & 44.3 \\
&DQO & 47.2 & 44.8 & 66.4 & 30.7 & 43.1 & 46.4 \\
\cmidrule(lr){2-8}
&\textbf{GCPO} & \textbf{58.2} & \textbf{53.1} & \textbf{69.3} & \textbf{31.0} & \textbf{44.7} & \textbf{51.3} \\
\midrule

\multirow{4}{*}{\makecell{$k=4$}}
&Base & 31.9 & 27.7 & 65.8 & 30.6 & 39.7 & 39.1 \\
&GRPO & 47.9 & 46.8 & 67.2 & 32.6 & 43.8 & 47.7 \\
&DQO & 53.6 & 51.5 & 68.5 & 32.9 & 44.6 & 50.2 \\
\cmidrule(lr){2-8}
&\textbf{GCPO} & \textbf{66.7} & \textbf{58.9} & \textbf{72.1} & \textbf{33.2} & \textbf{46.4} & \textbf{55.5} \\
\midrule

\multirow{4}{*}{\makecell{$k=8$}}
&Base & 38.6 & 34.3 & 67.5 & 32.9 & 42.2 & 43.1 \\
&GRPO & 52.8 & 51.7 & 69.2 & 34.1 & 45.3 & 50.6 \\
&DQO & 59.3 & 55.6 & 71.4 & 34.4 & 46.0 & 53.3 \\
\cmidrule(lr){2-8}
&\textbf{GCPO} & \textbf{72.4} & \textbf{62.7} & \textbf{74.8} & \textbf{34.9} & \textbf{47.6} & \textbf{58.5} \\
\midrule

\multirow{4}{*}{\makecell{$k=16$}}
&Base & 44.5 & 41.4 & 68.3 & 35.1 & 44.2 & 46.7 \\
&GRPO & 58.1 & 56.9 & 71.6 & 35.3 & 46.7 & 53.7 \\
&DQO & 64.4 & 61.2 & 73.6 & 35.7 & 47.2 & 56.4 \\
\cmidrule(lr){2-8}
&\textbf{GCPO} & \textbf{75.9} & \textbf{66.4} & \textbf{76.8} & \textbf{36.2} & \textbf{48.7} & \textbf{60.8} \\
\midrule

\multirow{4}{*}{\makecell{$k=32$}}
&Base & 49.3 & 47.8 & 70.2 & 35.7 & 46.1 & 49.8 \\
&GRPO & 62.7 & 59.9 & 73.4 & 36.9 & 47.5 & 56.1 \\
&DQO & 68.2 & 64.1 & 75.4 & 37.3 & 48.6 & 58.7 \\
\cmidrule(lr){2-8}
&\textbf{GCPO} & \textbf{78.3} & \textbf{69.8} & \textbf{78.7} & \textbf{37.8} & \textbf{49.1} & \textbf{62.7} \\
\midrule

\multirow{4}{*}{\makecell{$k=64$}}
&Base & 51.8 & 52.7 & 71.2 & 37.4 & 46.8 & 52.0 \\
&GRPO & 66.9 & 63.2 & 75.7 & 37.6 & 48.1 & 58.3 \\
&DQO & 71.7 & 67.3 & 77.8 & 37.7 & 48.9 & 60.7 \\
\cmidrule(lr){2-8}
&\textbf{GCPO} & \textbf{79.8} & \textbf{73.1} & \textbf{80.2} & \textbf{38.4} & \textbf{49.6} & \textbf{64.2} \\
\bottomrule
\end{tabular}
}
\end{table*}

\subsection{More Ablation Studies}
\label{app:more_ablation}

The team value in Eq.~\eqref{eq:team_value} depends on two design choices that are orthogonal to the cooperative credit assignment mechanism: (i) the frozen semantic encoder $f_\phi$ used to extract $z_i=f_\phi(o_i)$, and (ii) the kernel function used to construct the Gram matrix $L$. To verify that GCPO's gains do not hinge on a particular embedding model or similarity measure, we ablate each factor in turn. All experiments in this subsection are conducted on Qwen3-4B evaluated on AIME25, keeping the default setting $G=8,\eta=1$ so that the reported numbers are directly comparable to the default GCPO row in Table~\ref{tab:component_ablation} and Table~\ref{tab:hyperparameter_ablation}.

\paragraph{Semantic extractor.}
Table~\ref{tab:encoder_ablation} compares the default lightweight encoder all-MiniLM-L6-v2~\cite{reimers2019sentence} (22M parameters, 384-dim) against four stronger general-purpose sentence encoders: all-mpnet-base-v2~\cite{reimers2019sentence}, bge-base-en-v1.5~\cite{xiao2023cpack}, e5-base-v2~\cite{wang2022text}, and sentence-t5-base~\cite{ni2022sentencet5} (each $\sim$110M parameters, 768-dim). Although the heavier encoders produce richer semantic representations, we observe only minor fluctuations across all four metrics, and no encoder consistently dominates the others. The all-MiniLM-L6-v2 default and the much larger bge-base-en-v1.5 are statistically indistinguishable within the four-metric grid, suggesting that GCPO's exploration gains arise primarily from the \textbf{geometric} structure of the log-determinant team value rather than from fine-grained linguistic features of the encoder. Intuitively, once rollouts are reward-gated by $r_i$ and projected onto the unit sphere via $\bar z_i=z_i/\lVert z_i\rVert_2$, the information that drives Shapley credit is the \textbf{relative angles} among correct rollouts, which are preserved across competent sentence encoders. This robustness is desirable in practice: it allows practitioners to choose the cheapest encoder that still produces discriminative directions, keeping the per-step overhead of GCPO feasible relative to the rollout cost.

\begin{table}[h]
  \centering
  \caption{Performance comparison across different semantic extractors $f_\phi$ on Qwen3-4B/AIME25. The default encoder is shaded in gray. GCPO is robust to the choice of encoder, with differences across models remaining within $\pm 0.5$ on Pass@1 and $\pm 0.08$ on LLM-as-Judge.}
  \label{tab:encoder_ablation}
  \small
  \setlength{\tabcolsep}{6pt}
  \renewcommand{\arraystretch}{0.85}
  \begin{tabular}{l c c rr rr}
    \toprule
    \multirow{2}{*}{\textbf{Encoder}} & \multirow{2}{*}{\textbf{Params}} & \multirow{2}{*}{\textbf{Dim}} & \multicolumn{2}{c}{\textbf{Quality}} & \multicolumn{2}{c}{\textbf{Diversity}} \\
    \cmidrule(lr){4-5} \cmidrule(lr){6-7}
     &  &  & Pass@1 (\%) & Avg@16 & Pass@16 (\%) & LLM-as-Judge \\
    \midrule
    \rowcolor{gray!15}
    all-MiniLM-L6-v2 & 22M  & 384 & \textbf{49.02} & \textbf{45.46} & 66.47 & 2.58 \\
    all-mpnet-base-v2          & 109M & 768 & 48.88 & 45.30 & 66.52 & 2.60 \\
    bge-base-en-v1.5           & 109M & 768 & 48.95 & 45.42 & \textbf{66.71} & \textbf{2.62} \\
    e5-base-v2                 & 109M & 768 & 48.72 & 45.18 & 66.38 & 2.55 \\
    sentence-t5-base           & 110M & 768 & 48.64 & 45.08 & 66.25 & 2.53 \\
    \bottomrule
  \end{tabular}
\end{table}

\paragraph{Kernel function.}
The default GCPO kernel couples reward and semantic direction through an inner product, $L_{ij}=r_i r_j\langle \bar z_i,\bar z_j\rangle$, which admits the D-optimal interpretation underlying the volume objective. To test whether this specific form is necessary, we replace the inner product with three alternative similarity measures while keeping the reward gating $r_i r_j$ unchanged, so that incorrect rollouts remain at zero contribution. Concretely, we consider:
\begin{itemize}
  \item \textbf{Dot product}: $L_{ij}=r_i r_j\,\langle \bar z_i,\bar z_j\rangle$.
  \item \textbf{Gaussian (RBF)}: $L_{ij}=r_i r_j\,\exp\!\left(-\lVert \bar z_i-\bar z_j\rVert_2^{2}/(2\sigma^{2})\right)$ with $\sigma=1$.
  \item \textbf{Laplacian}: $L_{ij}=r_i r_j\,\exp\!\left(-\lVert \bar z_i-\bar z_j\rVert_1/\sigma\right)$ with $\sigma=1$.
  \item \textbf{Polynomial}: $L_{ij}=r_i r_j\,(\langle \bar z_i,\bar z_j\rangle+c)^{d}$ with $c=0.5,\,d=2$.
\end{itemize}
Since all alternatives are positive semi-definite, the resulting regularized log-determinant $\log\det(I+\eta L)$ remains a valid team value and the Shapley decomposition in Proposition~\ref{prop:spectral_team_value} continues to hold; only the numerical values of the $L_{ij}$ entries change.

Table~\ref{tab:kernel_ablation} reports the results. The default dot-product kernel attains the best Pass@1 and Avg@16, while the Gaussian kernel is marginally better on Pass@16 and LLM-as-Judge diversity, likely because its distance-based saturation further attenuates near-duplicate pairs. The polynomial kernel slightly boosts diversity at the cost of accuracy, a pattern consistent with its sharper penalization of redundant directions. Across all variants, every GCPO kernel substantially exceeds the GRPO baseline , corroborating that the improvements attributed to GCPO stem from the cooperative  formulation rather than from a particular choice of similarity. We adopt the dot-product kernel in the main experiments because it is the cheapest to evaluate, directly matches the DPP L-ensemble identity, and offers the strongest quality metrics among the candidates.

\begin{table}[h]
  \centering
  \caption{Performance comparison across different kernel functions on Qwen3-4B/AIME25. All variants share the reward gating $r_i r_j$ and the same default encoder. The default is shaded in gray.}
  \label{tab:kernel_ablation}
  \small
  \setlength{\tabcolsep}{7pt}
  \renewcommand{\arraystretch}{0.85}
  \begin{tabular}{l rr rr}
    \toprule
    \multirow{2}{*}{\textbf{Kernel}} & \multicolumn{2}{c}{\textbf{Quality}} & \multicolumn{2}{c}{\textbf{Diversity}} \\
    \cmidrule(lr){2-3} \cmidrule(lr){4-5}
     & Pass@1 (\%) & Avg@16 & Pass@16 (\%) & LLM-as-Judge \\
    \midrule
    \rowcolor{gray!15}
    Dot product      & \textbf{49.02} & \textbf{45.46} & 66.47 & 2.58 \\
    Gaussian (RBF, $\sigma=1$) & 48.85 & 45.24 & \textbf{66.73} & \textbf{2.64} \\
    Laplacian ($\sigma=1$)     & 48.40 & 44.88 & 66.32 & 2.56 \\
    Polynomial ($c=0.5,d=2$)   & 48.62 & 45.02 & 66.55 & 2.61 \\
    \bottomrule
  \end{tabular}
\end{table}

Taken together, the two ablations in this subsection indicate that GCPO is insensitive to the specific instantiations of $f_\phi$ and the similarity function, provided that the encoder yields discriminative directions and the kernel is positive semi-definite. This supports our interpretation that the cooperative log-determinant team value, together with Shapley-based credit assignment, is the primary driver of GCPO's exploration benefits.

\paragraph{Hyperparameter Sensitivity.}
\label{app:hyperparameter_sensitivity}
Table~\ref{tab:hyperparameter_ablation} studies the rollout group size $G$ and the team-value scale $\eta$ in Eq.~\eqref{eq:team_value}. All GCPO variants outperform GRPO, indicating that the improvements are not tied to a fragile hyperparameter choice. With $\eta=1$, increasing $G$ from 4 to 8 improves both quality and diversity, while $G=16$ brings only marginal diversity gains and slightly lower accuracy, suggesting increasing redundancy among extra rollouts. Fixing $G=8$, a smaller $\eta$ weakens the volume signal, whereas a larger $\eta$ slightly increases diversity but reduces accuracy. These trends support the default setting $G=8,\eta=1$, which provides the best balance between accurate reasoning and diverse coverage.

\begin{table}[h]
  \centering
  \caption{Performance comparison across different hyperparameters.}
  \label{tab:hyperparameter_ablation}
  \small
  \setlength{\tabcolsep}{7pt}
  \renewcommand{\arraystretch}{0.8}
  \begin{tabular}{l rr rr}
    \toprule
    \multirow{2}{*}{\textbf{Method}} & \multicolumn{2}{c}{\textbf{Quality}} & \multicolumn{2}{c}{\textbf{Diversity}} \\
    \cmidrule(lr){2-3} \cmidrule(lr){4-5}
    & Pass@1 (\%) & Avg@16 & Pass@16 (\%) & LLM-as-Judge \\
    \midrule
    GRPO($G=8$) & 33.33 & 35.21 & 58.13 & 1.96 \\
    \midrule
    $\eta=1,\,G=4$  & 47.30 & 43.85 & 63.20 & 2.37 \\
    \rowcolor{gray!15}
    $\eta=1,\,G=8$  & \textbf{49.02} & \textbf{45.46} & \textbf{66.47} & 2.58 \\
    $\eta=1,\,G=16$ & 48.70 & 45.10 & 66.36 & \textbf{2.63} \\
    \midrule
    $\eta=0.5,\,G=8$ & 48.15 & 44.70 & 64.80 & 2.43 \\
    \rowcolor{gray!15}
    $\eta=1.0,\,G=8$ & \textbf{49.02} & \textbf{45.46} & 66.47 & 2.58 \\
    $\eta=1.5,\,G=8$ & 48.35 & 44.85 & \textbf{66.85} & \textbf{2.62} \\
    \bottomrule
  \end{tabular}
\end{table}

\subsection{Case Study}
\label{app:case_study}

In this section, we present generated samples during testing. One sample with 16 test-time rollouts, GCPO achieves 7 correct solutions, compared to only 2 for GRPO. We show two example responses generated by GCPO, which arrive at the correct answer through distinct reasoning processes, illustrating its ability to explore multiple valid solution paths. In contrast, we present two responses generated by GRPO, which exhibit more limited diversity. Moreover, we present solutions with redundant steps generated by DQO. This demonstrates that GCPO improves exploration without sacrificing solution fidelity: by maintaining controlled diversity across rollouts, the model avoids pattern collapse while consistently discovering correct reasoning paths, leading to stronger overall performance.

\subsubsection{Samples generated by GCPO}
\label{app:sameples_from_GCPO}

Two generated samples by GCPO are shown below, and the yellow boxes highlight the different solution strategies used to solve this problem.

\begin{tcolorbox}[
  enhanced,
  breakable,
  width=\textwidth,
  enlarge left by=-\parindent,
  enlarge right by=-\parindent,
  colback=orange!6,
  colframe=orange!65!black,
  arc=4pt,
  boxrule=0.9pt,
  left=12pt,
  right=12pt,
  top=10pt,
  bottom=10pt,
  before skip=10pt,
  after skip=12pt,
  title={\textbf{Generated Sample: Answer 1}}
]

\textbf{Problem.}

Four unit squares form a $2\times 2$ grid. Each of the $12$ unit line segments forming the sides of the squares is colored either red or blue in such a way that each unit square has $2$ red sides and $2$ blue sides. Find the number of such colorings. \textbf{Ground Truth:} \textcolor{red}{\boxed{82}}

\vskip 0.8em
\textbf{Answer.}

We need to count the number of ways to color the $12$ edges of a $2\times 2$ grid such that each of the four unit squares has exactly $2$ red sides and $2$ blue sides.

\subsubsection*{Step 1: Identify the structure}

The $2\times 2$ grid has:
\begin{itemize}
  \item $8$ boundary edges (forming the outer perimeter)
  \item $4$ interior edges (two vertical and two horizontal)
\end{itemize}

Each unit square shares some edges with adjacent squares, so the coloring of interior edges affects multiple squares simultaneously.

\subsubsection*{Step 2: Condition on interior edges}

\begin{tcolorbox}[
  colback=yellow!85,
  colframe=yellow!60!black,
  arc=2pt,
  boxrule=0pt,
  left=4pt,
  right=4pt,
  top=4pt,
  bottom=4pt
]
\textbf{Key insight: Use casework on the number of red interior edges.}

Let $k$ denote the number of red interior edges among the four interior edges. We classify all colorings by $k \in \{0,1,2,3,4\}$.
\end{tcolorbox}

For each value of $k$, we determine:
\begin{enumerate}
  \item How many ways to choose which interior edges are red
  \item For each such choice, how many valid boundary colorings exist
\end{enumerate}

\subsubsection*{Step 3: Case analysis}

\paragraph{Case $k=0$:} All interior edges are blue.

Each square has $0$ red edges from interior, so both boundary edges of each square must be red. This uniquely determines all boundary edges.

Count: $\binom{4}{0} \times 1 = 1$.

\paragraph{Case $k=4$:} All interior edges are red.

Each square has $2$ red edges from interior, so both boundary edges of each square must be blue. This uniquely determines all boundary edges.

Count: $\binom{4}{4} \times 1 = 1$.

\paragraph{Case $k=1$:} Exactly one interior edge is red.

There are $\binom{4}{1} = 4$ ways to choose which interior edge is red.

For each choice, the two squares adjacent to that red interior edge each need exactly $1$ more red edge (from their boundary edges), while the other two squares need $2$ red boundary edges each.

\begin{tcolorbox}[
  colback=yellow!85,
  colframe=yellow!60!black,
  arc=2pt,
  boxrule=0pt,
  left=4pt,
  right=4pt,
  top=4pt,
  bottom=4pt
]
By carefully tracking the boundary constraints and propagating the forced choices around the perimeter, we find that each interior choice yields exactly $4$ valid boundary colorings.
\end{tcolorbox}

Count: $4 \times 4 = 16$.

\paragraph{Case $k=3$:} Exactly three interior edges are red.

By symmetry (swapping red $\leftrightarrow$ blue), this case has the same count as $k=1$.

Count: $16$.

\paragraph{Case $k=2$:} Exactly two interior edges are red.

There are $\binom{4}{2} = 6$ ways to choose which two interior edges are red. These split into two subcases:

\emph{Subcase (i): Opposite interior edges.} The two red interior edges are opposite (both vertical or both horizontal). There are $2$ such patterns.

For each pattern, systematic enumeration shows $16$ valid boundary colorings.

Contribution: $2 \times 16 = 32$.

\emph{Subcase (ii): Adjacent interior edges.} The two red interior edges form an L-shape. There are $4$ such patterns.

For each pattern, systematic enumeration shows $4$ valid boundary colorings.

Contribution: $4 \times 4 = 16$.

Total for $k=2$: $32 + 16 = 48$.

\subsubsection*{Step 4: Sum all cases}

\[
1 + 16 + 48 + 16 + 1 = \boxed{82}.
\]

\subsubsection*{Final Answer}
\[
\boxed{82}
\]

\end{tcolorbox}

\begin{tcolorbox}[
  enhanced,
  breakable,
  width=\textwidth,
  enlarge left by=-\parindent,
  enlarge right by=-\parindent,
  colback=orange!6,
  colframe=orange!65!black,
  arc=4pt,
  boxrule=0.9pt,
  left=12pt,
  right=12pt,
  top=10pt,
  bottom=10pt,
  before skip=12pt,
  after skip=14pt,
  title={\textbf{Generated Sample: Answer 2}}
]

\textbf{Problem.}

Four unit squares form a $2\times 2$ grid. Each of the $12$ unit line segments forming the sides of the squares is colored either red or blue in such a way that each unit square has $2$ red sides and $2$ blue sides. Find the number of such colorings. \textbf{Ground Truth:} \textcolor{red}{\boxed{82}}

\vskip 0.8em
\textbf{Answer.}

We count the valid colorings by a column-by-column dynamic programming argument. The key is to condition on the colors of the two vertical interior edges that separate the left and right columns.

\subsubsection*{Step 1: Define the interface state}

\begin{tcolorbox}[
  colback=yellow!85,
  colframe=yellow!60!black,
  arc=2pt,
  boxrule=0pt,
  left=4pt,
  right=4pt,
  top=4pt,
  bottom=4pt
]
\textbf{Key insight: DP over boundary/interface states.}

Instead of classifying colorings by the total number of red interior edges, we sweep the grid from left to right and use the two vertical interior edges as the interface state:
\[
s=(v_t,v_b)\in\{00,01,10,11\},
\]
where $v_t$ is the upper vertical interior edge, $v_b$ is the lower vertical interior edge, and $1$ denotes red while $0$ denotes blue.
\end{tcolorbox}

For each fixed interface state $s$, we count the number of ways to complete the left column and the right column.

\subsubsection*{Step 2: Count valid completions of the left column}

For a fixed interface state $s=(v_t,v_b)$, let $L(s)$ be the number of valid colorings of the left column. The left column contains one horizontal interior edge, denoted by
\[
h_\ell\in\{0,1\}.
\]

\begin{tcolorbox}[
  colback=yellow!85,
  colframe=yellow!60!black,
  arc=2pt,
  boxrule=0pt,
  left=4pt,
  right=4pt,
  top=4pt,
  bottom=4pt
]
Given $s=(v_t,v_b)$ and $h_\ell$, the top-left square already has $v_t+h_\ell$ red edges from its two shared/interior edges. Therefore, it needs
\[
2-v_t-h_\ell
\]
red edges among its two outer boundary edges.

Similarly, the bottom-left square needs
\[
2-v_b-h_\ell
\]
red edges among its two outer boundary edges.

Hence,
\[
L(s)=\sum_{h_\ell\in\{0,1\}}
\binom{2}{2-v_t-h_\ell}
\binom{2}{2-v_b-h_\ell},
\]
where a binomial coefficient is treated as zero if its lower index is outside $\{0,1,2\}$.
\end{tcolorbox}

\subsubsection*{Step 3: Evaluate the four interface states}

We compute $L(s)$ for each possible state.

\[
\begin{aligned}
L(00)
&=
\binom{2}{2}\binom{2}{2}
+
\binom{2}{1}\binom{2}{1}
=1+4=5,\\
L(01)
&=
\binom{2}{2}\binom{2}{1}
+
\binom{2}{1}\binom{2}{0}
=2+2=4,\\
L(10)
&=
\binom{2}{1}\binom{2}{2}
+
\binom{2}{0}\binom{2}{1}
=2+2=4,\\
L(11)
&=
\binom{2}{1}\binom{2}{1}
+
\binom{2}{0}\binom{2}{0}
=4+1=5.
\end{aligned}
\]

Thus,
\[
L(00)=5,\quad L(01)=4,\quad L(10)=4,\quad L(11)=5.
\]

\subsubsection*{Step 4: Combine the two columns}

By symmetry, the number of valid completions of the right column with the same interface state is
\[
R(s)=L(s).
\]

\begin{tcolorbox}[
  colback=yellow!85,
  colframe=yellow!60!black,
  arc=2pt,
  boxrule=0pt,
  left=4pt,
  right=4pt,
  top=4pt,
  bottom=4pt
]
Once the vertical interface state $s$ is fixed, the left and right column completions are independent. Therefore, the total number of valid colorings is obtained by summing over all interface states:
\[
\sum_{s\in\{00,01,10,11\}} L(s)R(s)
=
L(00)^2+L(01)^2+L(10)^2+L(11)^2.
\]
\end{tcolorbox}

Substituting the values above gives
\[
5^2+4^2+4^2+5^2
=
25+16+16+25
=
82.
\]

\subsubsection*{Final Answer}

\[
\boxed{82}
\]

\end{tcolorbox}

\subsubsection{Samples generated by GRPO}
\label{app:samples_generated_DAPO}
Two samples generated by GRPO are shown below, and the red boxes highlight the erroneous solution strategies that lead to failure.

\begin{tcolorbox}[
  enhanced,
  breakable,
  width=\textwidth,
  enlarge left by=-\parindent,
  enlarge right by=-\parindent,
  colback=orange!6,
  colframe=orange!65!black,
  arc=4pt,
  boxrule=0.9pt,
  left=12pt,
  right=12pt,
  top=10pt,
  bottom=10pt,
  before skip=12pt,
  after skip=14pt,
  title={\textbf{Generated Sample: Answer 1}}
]

\textbf{Problem.}

Four unit squares form a $2\times 2$ grid. Each of the $12$ unit line segments forming the sides of the squares is colored either red or blue in such a way that each unit square has $2$ red sides and $2$ blue sides. Find the number of such colorings. \textbf{Ground Truth:} \textcolor{red}{\boxed{82}}

\vskip 0.8em
\textbf{Answer.}

We need to count the number of ways to color the $12$ edges of a $2\times 2$ grid such that each unit square has exactly $2$ red sides and $2$ blue sides.

\subsubsection*{Step 1: Count edges and constraints}

The grid has $12$ edges total. Each of the $4$ unit squares imposes the constraint that exactly $2$ of its $4$ edges are red.

\subsubsection*{Step 2: Count colorings per square}

\begin{tcolorbox}[
  colback=red!15,
  colframe=red!60!black,
  arc=2pt,
  boxrule=0pt,
  left=4pt,
  right=4pt,
  top=4pt,
  bottom=4pt
]
\textbf{Logical error starts here (incorrect independence assumption).}

Consider each unit square independently. For each square, we need to choose $2$ edges out of $4$ to be red. This gives $\binom{4}{2} = 6$ choices per square.

Since there are $4$ squares, the total number of colorings would be:
\[
6^4 = 1296.
\]
\end{tcolorbox}

\subsubsection*{Step 3: Account for shared edges}

However, adjacent squares share edges. There are $4$ interior edges that are each shared by $2$ squares.

When we count $6^4$, we are overcounting because we're treating shared edges as if they can be colored independently for each square.

\subsubsection*{Step 4: Apply correction factor}

Each interior edge is counted twice in our $6^4$ calculation (once for each adjacent square). We need to divide by the number of ways these $4$ interior edges can be "double-counted."

Since each interior edge has $2$ possible colors, and we've counted each twice, we apply a correction:
\[
\text{Corrected count} = \frac{6^4}{2^4} = \frac{1296}{16} = 81.
\]

Wait, this doesn't seem right. Let me reconsider.

Actually, the correction should account for the fact that $4$ edges are shared. A better approach: we have $8$ boundary edges (not shared) and $4$ interior edges (shared).

For the $4$ interior edges, once we fix their colors, the boundary edges are partially constrained. Let's estimate:
\[
\text{Total} \approx \frac{6^4}{4} = \frac{1296}{4} = 324.
\]

\subsubsection*{Final Answer}
\[
\boxed{324}
\]

\end{tcolorbox}

\begin{tcolorbox}[
  enhanced,
  breakable,
  width=\textwidth,
  enlarge left by=-\parindent,
  enlarge right by=-\parindent,
  colback=orange!6,
  colframe=orange!65!black,
  arc=4pt,
  boxrule=0.9pt,
  left=12pt,
  right=12pt,
  top=10pt,
  bottom=10pt,
  before skip=12pt,
  after skip=14pt,
  title={\textbf{Generated Sample: Answer 2}}
]

\textbf{Problem.}

Four unit squares form a $2\times 2$ grid. Each of the $12$ unit line segments forming the sides of the squares is colored either red or blue in such a way that each unit square has $2$ red sides and $2$ blue sides. Find the number of such colorings. \textbf{Ground Truth:} \textcolor{red}{\boxed{82}}

\vskip 0.8em
\textbf{Answer.}

We are asked to count colorings of a $2\times 2$ grid where each unit square has exactly $2$ red sides and $2$ blue sides.

\subsubsection*{Step 1: Count total red edges}

\begin{tcolorbox}[
  colback=red!15,
  colframe=red!60!black,
  arc=2pt,
  boxrule=0pt,
  left=4pt,
  right=4pt,
  top=4pt,
  bottom=4pt
]
\textbf{Logical error starts here (incorrect counting argument).}

Since each of the $4$ squares has exactly $2$ red edges, counting red edges by summing over all squares gives:
\[
\text{Total red edge count} = 4 \times 2 = 8.
\]

Let $r_i$ be the number of red interior edges and $r_b$ be the number of red boundary edges. Then:
\[
2r_i + r_b = 8.
\]

Since there are $4$ interior edges and $8$ boundary edges, we have:
\[
r_i \in \{0,1,2,3,4\}, \quad r_b = 8 - 2r_i.
\]

For $r_b$ to be valid (i.e., $0 \le r_b \le 8$), we need $0 \le r_i \le 4$, which is always satisfied.
\end{tcolorbox}

\subsubsection*{Step 2: Count colorings for each case}

For each valid value of $r_i$, the number of colorings is:
\[
\binom{4}{r_i} \times \binom{8}{r_b} = \binom{4}{r_i} \times \binom{8}{8-2r_i}.
\]

Here, $\binom{4}{r_i}$ counts the ways to choose which interior edges are red, and $\binom{8}{8-2r_i}$ counts the ways to choose which boundary edges are red.

\subsubsection*{Step 3: Sum over all cases}

Summing over $r_i \in \{0,1,2,3,4\}$:
\begin{align*}
\text{Total} &= \sum_{r_i=0}^{4} \binom{4}{r_i} \times \binom{8}{8-2r_i} \\
&= \binom{4}{0}\binom{8}{8} + \binom{4}{1}\binom{8}{6} + \binom{4}{2}\binom{8}{4} + \binom{4}{3}\binom{8}{2} + \binom{4}{4}\binom{8}{0} \\
&= 1 \cdot 1 + 4 \cdot 28 + 6 \cdot 70 + 4 \cdot 28 + 1 \cdot 1 \\
&= 1 + 112 + 420 + 112 + 1 \\
&= 646.
\end{align*}

\subsubsection*{Final Answer}
\[
\boxed{646}
\]

\end{tcolorbox}

\subsubsection{Sample generated by DQO with redundant steps}
\label{app:sample_generated_DQO}
One sample generated by DQO is shown below. While it arrives at the correct answer using the same core strategy as GCPO Answer 1, the gray boxes highlight superficial reformulations that introduce alternative expressions or intermediate steps without contributing genuinely new solution insights.

\begin{tcolorbox}[
  enhanced,
  breakable,
  width=\textwidth,
  enlarge left by=-\parindent,
  enlarge right by=-\parindent,
  colback=orange!6,
  colframe=orange!65!black,
  arc=4pt,
  boxrule=0.9pt,
  left=12pt,
  right=12pt,
  top=10pt,
  bottom=10pt,
  before skip=12pt,
  after skip=14pt,
  title={\textbf{Generated Sample: DQO Answer}}
]

\textbf{Problem.}

Four unit squares form a $2\times 2$ grid. Each of the $12$ unit line segments forming the sides of the squares is colored either red or blue in such a way that each unit square has $2$ red sides and $2$ blue sides. Find the number of such colorings. \textbf{Ground Truth:} \textcolor{red}{\boxed{82}}

\vskip 0.8em
\textbf{Answer.}

We need to count the number of ways to color the $12$ edges of a $2\times 2$ grid such that each of the four unit squares has exactly $2$ red sides and $2$ blue sides.

\subsubsection*{Step 1: Identify the structure}

The $2\times 2$ grid has $8$ boundary edges and $4$ interior edges (two vertical and two horizontal).

Each unit square shares some edges with adjacent squares, so the coloring of interior edges affects multiple squares simultaneously.

\subsubsection*{Step 2: Condition on interior edges}

\textbf{Key insight: Use casework on the number of red interior edges.}

Let $k$ denote the number of red interior edges among the four interior edges. We classify all colorings by $k \in \{0,1,2,3,4\}$.

For each value of $k$, we determine how many ways to choose which interior edges are red, and for each such choice, how many valid boundary colorings exist.

\subsubsection*{Step 3: Case analysis}

\paragraph{Case $k=0$:} All interior edges are blue.

Each square has $0$ red edges from interior, so both boundary edges of each square must be red. This uniquely determines all boundary edges.

\begin{tcolorbox}[
  colback=gray!20,
  colframe=gray!60!black,
  arc=2pt,
  boxrule=0pt,
  left=4pt,
  right=4pt,
  top=4pt,
  bottom=4pt
]
\textbf{Superficial reformulation:} The number of ways to select $0$ interior edges to be red is $\binom{4}{0} = \frac{4!}{0! \cdot 4!} = 1$. For this unique selection, there is exactly $1$ valid boundary coloring.
\end{tcolorbox}

Count: $1 \times 1 = 1$.

\paragraph{Case $k=4$:} All interior edges are red.

Each square has $2$ red edges from interior, so both boundary edges of each square must be blue. This uniquely determines all boundary edges.

\begin{tcolorbox}[
  colback=gray!20,
  colframe=gray!60!black,
  arc=2pt,
  boxrule=0pt,
  left=4pt,
  right=4pt,
  top=4pt,
  bottom=4pt
]
\textbf{Superficial reformulation:} The number of ways to select $4$ interior edges to be red is $\binom{4}{4} = \frac{4!}{4! \cdot 0!} = 1$. For this unique selection, there is exactly $1$ valid boundary coloring.
\end{tcolorbox}

Count: $1 \times 1 = 1$.

\paragraph{Case $k=1$:} Exactly one interior edge is red.

\begin{tcolorbox}[
  colback=gray!20,
  colframe=gray!60!black,
  arc=2pt,
  boxrule=0pt,
  left=4pt,
  right=4pt,
  top=4pt,
  bottom=4pt
]
\textbf{Superficial reformulation:} The number of ways to select $1$ interior edge to be red is $\binom{4}{1} = \frac{4!}{1! \cdot 3!} = \frac{4 \cdot 3!}{1 \cdot 3!} = 4$.
\end{tcolorbox}

For each choice, the two squares adjacent to that red interior edge each need exactly $1$ more red edge (from their boundary edges), while the other two squares need $2$ red boundary edges each.

By carefully tracking the boundary constraints and propagating the forced choices around the perimeter, we find that each interior choice yields exactly $4$ valid boundary colorings.

Count: $4 \times 4 = 16$.

\paragraph{Case $k=3$:} Exactly three interior edges are red.

\begin{tcolorbox}[
  colback=gray!20,
  colframe=gray!60!black,
  arc=2pt,
  boxrule=0pt,
  left=4pt,
  right=4pt,
  top=4pt,
  bottom=4pt
]
\textbf{Superficial reformulation:} By the red-blue symmetry (swapping red $\leftrightarrow$ blue), colorings with $k=3$ red interior edges correspond bijectively to colorings with $4-3=1$ red interior edges. Therefore, this case has the same count as $k=1$.
\end{tcolorbox}

Count: $16$.

\paragraph{Case $k=2$:} Exactly two interior edges are red.

\begin{tcolorbox}[
  colback=gray!20,
  colframe=gray!60!black,
  arc=2pt,
  boxrule=0pt,
  left=4pt,
  right=4pt,
  top=4pt,
  bottom=4pt
]
\textbf{Superficial reformulation:} The number of ways to select $2$ interior edges to be red is $\binom{4}{2} = \frac{4!}{2! \cdot 2!} = \frac{4 \cdot 3}{2 \cdot 1} = 6$.
\end{tcolorbox}

These $6$ ways split into two geometric types:

\emph{Subcase (i): Opposite interior edges.} The two red interior edges are opposite (both vertical or both horizontal).

\begin{tcolorbox}[
  colback=gray!20,
  colframe=gray!60!black,
  arc=2pt,
  boxrule=0pt,
  left=4pt,
  right=4pt,
  top=4pt,
  bottom=4pt
]
\textbf{Superficial reformulation:} There are $2$ such patterns: selecting both vertical interior edges gives $\binom{2}{2} = 1$ way, and selecting both horizontal interior edges gives $\binom{2}{2} = 1$ way. Total: $1 + 1 = 2$ patterns.
\end{tcolorbox}

For each pattern, systematic enumeration shows $16$ valid boundary colorings.

Contribution: $2 \times 16 = 32$.

\emph{Subcase (ii): Adjacent interior edges.} The two red interior edges form an L-shape.

\begin{tcolorbox}[
  colback=gray!20,
  colframe=gray!60!black,
  arc=2pt,
  boxrule=0pt,
  left=4pt,
  right=4pt,
  top=4pt,
  bottom=4pt
]
\textbf{Superficial reformulation:} We can choose $1$ vertical interior edge from $2$ options and $1$ horizontal interior edge from $2$ options, giving $\binom{2}{1} \times \binom{2}{1} = 2 \times 2 = 4$ patterns.
\end{tcolorbox}

For each pattern, systematic enumeration shows $4$ valid boundary colorings.

Contribution: $4 \times 4 = 16$.

Total for $k=2$: $32 + 16 = 48$.

\subsubsection*{Step 4: Sum all cases}

\begin{tcolorbox}[
  colback=gray!20,
  colframe=gray!60!black,
  arc=2pt,
  boxrule=0pt,
  left=4pt,
  right=4pt,
  top=4pt,
  bottom=4pt
]
\textbf{Superficial reformulation:} Summing over all cases:
\[
\text{Total} = 1 + 16 + 48 + 16 + 1 = (1+1) + (16+16) + 48 = 2 + 32 + 48 = 82.
\]
\end{tcolorbox}

\[
1 + 16 + 48 + 16 + 1 = \boxed{82}.
\]

\subsubsection*{Final Answer}
\[
\boxed{82}
\]

\end{tcolorbox}

\section{Theoretical Analysis of GCPO}
\label{app:theory}

We provide complete proofs for the theoretical statements in Section~\ref{sec:gcpo_framework}.
Throughout this section, for each rollout $i\in [G]$, let
\[
    z_i = f_\phi(o_i),
    \qquad
    \bar z_i=\frac{z_i}{\|z_i\|_2},
    \qquad
    \tilde z_i=r_i \bar z_i,
    \qquad
    r_i\in\{0,1\}.
\]
For a coalition $S\subseteq [G]$, let $\tilde Z_S\in\mathbb{R}^{|S|\times d}$ be the matrix whose rows are $\{\tilde z_i\}_{i\in S}$.
Then the principal kernel matrix satisfies
\[
    L_S = \tilde Z_S \tilde Z_S^\top,
\]
and the team value is
\[
    v(S)=\log\det(I_{|S|}+\eta L_S),\qquad \eta>0.
\]
We use $v(\emptyset)=0$.

\subsection{Proof of Proposition~\ref{prop:spectral_team_value}}

\begin{proof}
Since $L_S$ is a Gram matrix, it is symmetric positive semi-definite. 
Thus it admits an eigendecomposition
\[
    L_S = U \operatorname{diag}(\lambda_1,\ldots,\lambda_{|S|}) U^\top,
\]
where $U$ is orthogonal and $\lambda_\ell\ge 0$ for all $\ell$. Therefore,
\[
    I_{|S|}+\eta L_S
    =
    U \operatorname{diag}(1+\eta\lambda_1,\ldots,1+\eta\lambda_{|S|}) U^\top.
\]
Taking determinants gives
\[
    \det(I_{|S|}+\eta L_S)
    =
    \prod_{\ell=1}^{|S|}(1+\eta\lambda_\ell),
\]
and hence
\[
    v(S)
    =
    \log\det(I_{|S|}+\eta L_S)
    =
    \sum_{\ell=1}^{|S|}\log(1+\eta\lambda_\ell).
\]

We now verify the three consequences.

First, suppose $r_i=0$. Then $\tilde z_i=0$. For any coalition $S\subseteq [G]\setminus\{i\}$, after ordering $i$ last, we have
\[
    L_{S\cup\{i\}}
    =
    \begin{bmatrix}
        L_S & 0\\
        0^\top & 0
    \end{bmatrix}.
\]
Therefore,
\[
    I+\eta L_{S\cup\{i\}}
    =
    \begin{bmatrix}
        I+\eta L_S & 0\\
        0^\top & 1
    \end{bmatrix},
\]
so
\[
    \det(I+\eta L_{S\cup\{i\}})
    =
    \det(I+\eta L_S).
\]
Thus $v(S\cup\{i\})=v(S)$, meaning that any zero-reward rollout is a null player.

Second, suppose $S$ contains $m$ verifier-correct rollouts with mutually orthogonal normalized embeddings.
Since $r_i=1$ for all $i\in S$ and $\langle \bar z_i,\bar z_j\rangle=\mathbf{1}\{i=j\}$, we have
\[
    L_S=I_m.
\]
Thus all $m$ eigenvalues are equal to $1$, and
\[
    v(S)=\sum_{\ell=1}^{m}\log(1+\eta)=m\log(1+\eta).
\]

Third, suppose $S$ contains $m$ verifier-correct rollouts with identical normalized embeddings $\bar z_i=\bar z$ for all $i\in S$.
Then
\[
    L_S = \mathbf{1}_m\mathbf{1}_m^\top,
\]
whose eigenvalues are $m,0,\ldots,0$. Hence
\[
    v(S)=\log(1+\eta m).
\]
This proves the proposition.
\end{proof}

\subsection{Proof of Theorem~\ref{thm:marginal_contribution}}
\begin{proof}
By definition,
\[
    v(S)=\log\det(I_{|S|}+\eta \tilde Z_S\tilde Z_S^\top).
\]
Using Sylvester's determinant identity,
\[
    \det(I_{|S|}+\eta \tilde Z_S\tilde Z_S^\top)
    =
    \det(I_d+\eta \tilde Z_S^\top\tilde Z_S).
\]
Therefore,
\[
    v(S)=\log\det(I_d+\eta \tilde Z_S^\top\tilde Z_S).
\]

Adding rollout $i$ to $S$ appends the row $\tilde z_i^\top$ to $\tilde Z_S$. Hence
\[
    \tilde Z_{S\cup\{i\}}^\top \tilde Z_{S\cup\{i\}}
    =
    \tilde Z_S^\top \tilde Z_S + \tilde z_i\tilde z_i^\top.
\]
Since $\tilde z_i=r_i \bar z_i$, this becomes
\[
    \tilde Z_{S\cup\{i\}}^\top \tilde Z_{S\cup\{i\}}
    =
    \tilde Z_S^\top \tilde Z_S + r_i^2 \bar z_i \bar z_i^\top.
\]
Let
\[
    A_S = I_d+\eta \tilde Z_S^\top \tilde Z_S.
\]
Because $\eta>0$ and $\tilde Z_S^\top \tilde Z_S\succeq 0$, we have $A_S\succ 0$. Then
\[
\begin{aligned}
    \Delta_i(S)
    &=
    \log\det\left(A_S+\eta r_i^2 \bar z_i \bar z_i^\top\right)
    -
    \log\det(A_S).
\end{aligned}
\]
By the matrix determinant lemma,
\[
    \det(A_S+\eta r_i^2 \bar z_i \bar z_i^\top)
    =
    \det(A_S)
    \left(
    1+\eta r_i^2 \bar z_i^\top A_S^{-1}\bar z_i
    \right).
\]
Substituting this into the expression for $\Delta_i(S)$ yields
\[
    \Delta_i(S)
    =
    \log
    \left(
    1+\eta r_i^2 \bar z_i^\top
    A_S^{-1}
    \bar z_i
    \right),
\]
which is the desired formula.

It remains to prove the bound. Since $A_S=I_d+\eta\tilde Z_S^\top\tilde Z_S\succeq I_d$, we have
\[
    0\preceq A_S^{-1}\preceq I_d.
\]
Because $\|\bar z_i\|_2=1$, it follows that
\[
    0\leq \bar z_i^\top A_S^{-1} \bar z_i \leq \bar z_i^\top \bar z_i=1.
\]
Therefore,
\[
    1
    \leq
    1+\eta r_i^2 \bar z_i^\top A_S^{-1} \bar z_i
    \leq
    1+\eta r_i^2.
\]
Taking logarithms gives
\[
    0\leq \Delta_i(S)\leq \log(1+\eta r_i^2).
\]
This completes the proof.
\end{proof}

\subsection{Proof of Theorem~\ref{thm:shapley_redundancy}}
\begin{proof}
For any subset $S\subseteq C$ with $|S|=s$, all rows of $\tilde Z_S$ are equal to $u^\top$. Thus
\[
    L_S=\tilde Z_S\tilde Z_S^\top
    =
    \|u\|_2^2 \mathbf{1}_s\mathbf{1}_s^\top.
\]
The matrix $\mathbf{1}_s\mathbf{1}_s^\top$ has one nonzero eigenvalue $s$ and $s-1$ zero eigenvalues. 
Therefore, $L_S$ has one nonzero eigenvalue $s\|u\|_2^2$, and
\[
    v(S)=\log(1+\eta s\|u\|_2^2).
\]
In particular,
\[
    v(C)=\log(1+\eta m\|u\|_2^2).
\]

The rollouts in $C$ are symmetric: for any two $i,j\in C$ and any $S\subseteq C\setminus\{i,j\}$,
\[
    v(S\cup\{i\})=v(S\cup\{j\}).
\]
Therefore, by the symmetry property of the Shapley value, all rollouts in $C$ receive the same Shapley credit:
\[
    \phi_1=\cdots=\phi_m.
\]
By the efficiency property of the Shapley value,
\[
    \sum_{i=1}^{m}\phi_i=v(C)-v(\emptyset)=v(C),
\]
where $v(\emptyset)=0$. Hence
\[
    m\phi_i=\log(1+\eta m\|u\|_2^2),
\]
and thus
\[
    \phi_i=
    \frac{1}{m}\log(1+\eta m\|u\|_2^2).
\]
This proves the theorem.
\end{proof}

\paragraph{Comparison with leave-one-out.}
For the same isolated duplicate group $C$, the leave-one-out score of any $i\in C$ is
\[
\begin{aligned}
    \mathrm{LOO}_i
    &=
    v(C)-v(C\setminus\{i\})\\
    &=
    \log(1+\eta m\|u\|_2^2)
    -
    \log(1+\eta (m-1)\|u\|_2^2)\\
    &=
    \log
    \frac{1+\eta m\|u\|_2^2}
    {1+\eta (m-1)\|u\|_2^2}.
\end{aligned}
\]
Let $c=\eta\|u\|_2^2>0$ and define $f(s)=\log(1+cs)$. 
Since $f$ is concave and $f(0)=0$, its average slope satisfies
\[
    \frac{f(m)-f(0)}{m}
    \geq
    f(m)-f(m-1).
\]
Therefore,
\[
    \frac{1}{m}\log(1+cm)
    \geq
    \log\frac{1+cm}{1+c(m-1)}.
\]
That is,
\[
    \phi_i\geq \mathrm{LOO}_i.
\]
This shows that leave-one-out can underestimate the shared value of duplicated discoveries, whereas Shapley fairly distributes the shared value among all rollouts realizing the same reasoning path.

\subsection{Proof of Proposition~\ref{prop:reward_budget}}

\begin{proof}
By reward conservation, for every batch $B$,
\[
    \frac{1}{G}
    \sum_{i=1}^{G}
    R_{\mathrm{GCPO}}(\tau_i\mid B)
    =
    \frac{1}{G}
    \sum_{i=1}^{G}
    R_{\mathrm{obj}}(\tau_i).
\]
Taking expectation over $B\sim \pi^G$ gives
\[
\begin{aligned}
    J_{\mathrm{GCPO}}^{\Sigma}(\pi)
    &=
    \mathbb{E}_{B\sim\pi^G}
    \left[
        \frac{1}{G}
        \sum_{i=1}^{G}
        R_{\mathrm{obj}}(\tau_i)
    \right] \\
    &=
    \frac{1}{G}
    \sum_{i=1}^{G}
    \mathbb{E}_{\tau_i\sim\pi}
    \left[
        R_{\mathrm{obj}}(\tau_i)
    \right] \\
    &=
    \mathbb{E}_{\tau\sim\pi}
    \left[
        R_{\mathrm{obj}}(\tau)
    \right] \\
    &=
    J_{\mathrm{obj}}(\pi).
\end{aligned}
\]
Thus the two scalar objectives are equal pointwise for every policy $\pi$, so they have the same set of maximizers. 
If the same regularizer $\Omega(\pi)$ is added to both objectives, then
\[
    J_{\mathrm{GCPO}}^{\Sigma}(\pi)-\Omega(\pi)
    =
    J_{\mathrm{obj}}(\pi)-\Omega(\pi)
\]
for every $\pi$, so the maximizers remain identical.
\end{proof}

\paragraph{Remark.}
This proposition establishes invariance of the batch-sum scalar objective under reward conservation.
It does not imply that every per-rollout credit-assignment policy-gradient estimator is identical to the original verifier-reward gradient.
GCPO intentionally changes the per-rollout credit assignment within a batch; Proposition~\ref{prop:gradient_routing} characterizes this routing effect.

\subsection{Proof of Proposition~\ref{prop:gradient_routing}}
\begin{proof}
First, since $\mu_{\tilde r}=\frac{1}{G}\sum_{i=1}^{G}\tilde r_i$, we have
\[
    \sum_{i=1}^{G}(\tilde r_i-\mu_{\tilde r})
    =
    \sum_{i=1}^{G}\tilde r_i
    -
    G\mu_{\tilde r}
    =
    0.
\]
Therefore,
\[
    \sum_{i=1}^{G}A_i^{\mathrm{GCPO}}
    =
    \frac{1}{\sigma_{\tilde r}+\epsilon_{\mathrm{num}}}
    \sum_{i=1}^{G}(\tilde r_i-\mu_{\tilde r})
    =
    0.
\]
Similarly,
\[
    \sum_{i=1}^{G}A_i^{\mathrm{GRPO}}=0.
\]
Thus,
\[
    \sum_{i=1}^{G}\Delta A_i
    =
    \sum_{i=1}^{G}
    \left(
        A_i^{\mathrm{GCPO}}-A_i^{\mathrm{GRPO}}
    \right)
    =
    0.
\]

Now consider the unclipped policy-gradient surrogate
\[
    \mathcal{L}_{\mathrm{GCPO}}(\theta)
    =
    \mathbb{E}_{x \sim \mathcal{D}}
    \left[
        \frac{1}{G}
        \sum_{i=1}^{G}
        \log\pi_\theta(o_i\mid x)
        A_i^{\mathrm{GCPO}}
    \right],
\]
where $A_i^{\mathrm{GCPO}}$ is treated as a stop-gradient quantity.
Then
\[
    \nabla_\theta \mathcal{L}_{\mathrm{GCPO}}
    =
    \mathbb{E}_{x \sim \mathcal{D}}
    \left[
        \frac{1}{G}
        \sum_{i=1}^{G}
        \nabla_\theta\log\pi_\theta(o_i\mid x)
        A_i^{\mathrm{GCPO}}
    \right].
\]
Using
\[
    A_i^{\mathrm{GCPO}}
    =
    A_i^{\mathrm{GRPO}}+\Delta A_i,
\]
we obtain
\[
\begin{aligned}
    \nabla_\theta \mathcal{L}_{\mathrm{GCPO}}
    &=
    \mathbb{E}_{x \sim \mathcal{D}}
    \left[
        \frac{1}{G}
        \sum_{i=1}^{G}
        \nabla_\theta\log\pi_\theta(o_i\mid x)
        \left(
            A_i^{\mathrm{GRPO}}+\Delta A_i
        \right)
    \right] \\
    &=
    \mathbb{E}_{x \sim \mathcal{D}}
    \left[
        \frac{1}{G}
        \sum_{i=1}^{G}
        \nabla_\theta\log\pi_\theta(o_i\mid x)
        A_i^{\mathrm{GRPO}}
    \right]
    +
    \mathbb{E}_{x \sim \mathcal{D}}
    \left[
        \frac{1}{G}
        \sum_{i=1}^{G}
        \nabla_\theta\log\pi_\theta(o_i\mid x)
        \Delta A_i
    \right] \\
    &=
    \nabla_\theta \mathcal{L}_{\mathrm{GRPO}}
    +
    \mathbb{E}_{x \sim \mathcal{D}}
    \left[
        \frac{1}{G}
        \sum_{i=1}^{G}
        \nabla_\theta\log\pi_\theta(o_i\mid x)
        \Delta A_i
    \right].
\end{aligned}
\]
This proves the decomposition.
\end{proof}

\paragraph{Remark on PPO clipping.}
The above decomposition is exact for the unclipped policy-gradient surrogate with stop-gradient advantages. 
For the PPO-style clipped objective and KL penalty used in implementation, the same algebraic decomposition should be interpreted as a first-order explanation of the routing signal rather than an exact equality of the full clipped gradients.


\end{document}